\documentclass{article}

\usepackage[preprint]{neurips_2026}

\usepackage[utf8]{inputenc} 
\usepackage[T1]{fontenc}    
\usepackage{hyperref}       
\usepackage{url}            
\usepackage{booktabs}       
\usepackage{amsfonts}       
\usepackage{nicefrac}       
\usepackage{microtype}      
\usepackage{xcolor}         

\usepackage{microtype}
\usepackage{graphicx}
\usepackage{booktabs} 
\usepackage{multirow} 
\usepackage{enumitem}
\usepackage{hyperref}
\usepackage{algorithm}
\usepackage{algpseudocode} 
\usepackage{graphicx}
\usepackage{subcaption}

\usepackage{amsmath}
\usepackage{amssymb}
\usepackage{mathtools}
\usepackage{amsthm}

\usepackage[capitalize,noabbrev]{cleveref}

\theoremstyle{plain}
\newtheorem{theorem}{Theorem}[section]
\newtheorem{proposition}[theorem]{Proposition}

\theoremstyle{definition}

\theoremstyle{remark}

\usepackage{physics}
\usepackage{wrapfig}
\usepackage{subcaption}
\usepackage{xspace}
\DeclareRobustCommand{\model}{SBFM\xspace}
\usepackage[textsize=tiny]{todonotes}

\title{Accelerated Sequential Flow Matching:\\ A Bayesian Filtering Perspective}

\author{%
  Yinan Huang\\
  Georgia Institute of Technology\\
  \texttt{yhuang903@gatech.edu} \\
  \And
 Hans Hao-Hsun Hsu\\
  Georgia Institute of Technology\\
  \texttt{hans.hsu@gatech.edu}\\
  \And
  Juran Wang\\ 
  Georgia Institute of Technology\\
  \texttt{jwang3668@gatech.edu}\\
  \And
  Bo Dai\\ 
  Georgia Institute of Technology\\
  \texttt{bodai@cc.gatech.edu}\\
  \And 
  Pan Li\\
  Georgia Institute of Technology\\
  \texttt{panli@gatech.edu}\\
}

\begin{document}

\maketitle

\begin{abstract}
  Sequential probabilistic inference from streaming observations requires modeling distributions over future trajectories as new observations arrive. Although diffusion and flow-matching models are effective at capturing high-dimensional, multimodal distributions, their deployment in real-time streaming settings typically relies on repeatedly sampling from a non-informative initial distribution. This results in substantial inference latency, particularly when multiple samples are needed to characterize the predictive distribution. In this work, we introduce \emph{Sequential Bayesian Flow Matching}, a framework inspired by Bayesian filtering. By learning a probability flow that transports the posterior distribution from one time step to the next time step conditioned on new observations, it mirrors the recursive structure of Bayesian belief updates. Crucially, by using the previous belief as an informative source distribution, it enables substantially faster sampling than na\"ive resampling from scratch. Across scientific forecasting tasks spanning accelerator beam spill dynamics, fluid dynamics, and weather forecasting, as well as decision-making benchmarks, our method achieves performance competitive with full-step diffusion on distributional metrics while using far fewer sampling steps, substantially reducing inference latency. Our code is available at \url{https://github.com/Graph-COM/Sequential\_Flow\_Matching}. 
\end{abstract}
\section{Introduction}

Modeling the evolution of real-world dynamic systems from streaming observations is a fundamental problem, with applications in weather forecasting~\citep{kalnay2003atmospheric,lorenz2017deterministic}, fluid dynamics modeling~\citep{bewley2001flow}, magnetic plasma control~\citep{degrave2022magnetic}, and video forecasting for decision making in autonomous systems~\citep{yang2024generalized}. A critical characteristic of these systems is their inherent stochasticity: a single history of observations can admit multiple plausible futures. Deterministic predictions (such as simple regression) are often insufficient, as they tend to average out distinct possibilities, leading to physically invalid or blurred predictions. Moreover, many applications require inferring multi-step future trajectories for monitoring and planning. Ultimately, reliable decision-making requires moving beyond point estimation to maintain a probabilistic belief over the system's future trajectories as new observations arrive.


Probability flow-based models, such as diffusion models~\citep{sohl2015deep, song2019generative, ho2020denoising} and flow matching~\citep{lipman2023flow, liu2023flow}, have become powerful frameworks for modeling high-dimensional, multi-modal distributions. Their ability to capture complex uncertainty and generate high-fidelity samples makes them particularly effective for representing the distribution over long-horizon trajectories. In fact, there is growing interest in applying these models to probabilistic time-series and sequential modeling, including forecasting and state estimation to planning and control~\citep{chen2024diffusion,gao2023prediff,janner2022planning,wei2024diffphycon}.

However, deploying these models in real-time streaming settings remains challenging. Continuously incorporating streaming observations requires rapid updates to avoid delayed estimates. Existing approaches typically re-sample future trajectories from scratch to incorporate new observations, requiring tens to hundreds of network evaluations per update and incurring prohibitive inference latency~\citep{janner2022planning, chen2024diffusion, zhou2025diffusion}. The problem becomes even more acute when multiple samples are needed to characterize the trajectory distribution. Recent efforts seek to mitigate this cost through heuristics such as asynchronous denoising or the reuse of historical predictions~\citep{janner2022planning, duan2025real,li2026step,wei2025cldiffphycon, hoeg2025fast}. However, these approaches emphasize efficient inference pipelines rather than accurate posterior modeling, which can compromise long-horizon predictive performance (see discussion in Sections~\ref{sec:sbfm} and~\ref{sec:exp}).


In this paper, we study flow-based sequential probabilistic inference inspired by Bayesian filtering, and aim to efficiently adapt posterior distributions of future trajectories to new observations. Our central observation is that many streaming tasks exhibit strong \emph{temporal coherence}: the previous posterior $p(x_{t-1}|z_{\le t-1})$ at time $t-1$ is closely aligned with the updated posterior $p(x_t|z_{\le t})$ at time $t$, and thus serves as a natural prior for the next inference step. Exploiting this informative prior can effectively reduce the number of sampling steps required by flow-based models. Motivated by this perspective, we propose \textbf{Sequential Bayesian Flow Matching} (\model), a framework that explicitly parameterizes the posterior update process. Rather than discarding previous generations and resampling from a non-informative base distribution, \model maintains a set of particles representing the current posterior and evolves them as new observations arrive. These particles are transported toward regions consistent with the updated posterior, while modes that become inconsistent with the new observations are progressively filtered out, as illustrated later in Figure~\ref{fig:particle_traj}. This particle-based view is closely related to particle flow filtering~\citep{daum2010exact, daum2011particle, khan2014non}. The key distinction is that our flow is learned directly from data rather than derived from explicit system dynamics and observation models.

However, directly learning posterior-to-posterior flows from raw data requires some caution. A tempting idea is to train a flow matching model to transport between successive states $(x_{t-1},x_t)$ from the observed trajectory conditioned on the observations $z_{\le t}$. But this construction does not faithfully realize Bayesian filtering: once
conditioned on $z_{\le t}$, the state $x_{t-1}$ taken from an observed trajectory is effectively a sample from $p(x_{t-1}|z_{\le t})$, not from the posterior $p(x_{t-1}|z_{\le t-1})$ before assimilating the new observation at time $t$. This distributional mismatch is discussed formally in Appendix~\ref{appx:proof} and ablated empirically in Appendix~\ref{appx:ablation}.

To address this, we propose a simple pretraining-finetuning strategy that
adapts a pretrained Gaussian-to-posterior flow model to learn posterior-to-posterior
transports. We use the pretrained model to generate source samples from
the previous posterior $p(x_{t-1}|z_{\le t-1})$, and finetune \model (initialized from the pretrained weights) to transport these samples to the
ground-truth samples $x_t$.  
This strategy echoes the common practice of finetuning pretrained diffusion models for
sampling acceleration, such as progressive diffusion distillation and consistency
distillation~\citep{salimans2022progressive,song2023consistency}, although our
goal is to learn sequential belief updates rather than to compress the original
Gaussian-to-target sampling path.

We first evaluate our method on scientific applications. As the primary scientific driver, we consider beam spill dynamics in the Mu2e particle accelerator system~\citep{whitbeck2025fast}, where particles are extracted over time and the beam intensity evolves stochastically, requiring real-time, ultra-low-latency monitoring. Our method improves distributional accuracy in forecasting future beam intensity while substantially reducing inference latency through fewer sampling steps. Similar efficiency gains are observed on fluid dynamics and weather forecasting tasks. We further demonstrate the applicability of \model to planning and control settings, where it likewise enables low-latency closed-loop inference.




\section{Preliminary: Flow Matching and Diffusion Models}
\begin{table*}[t!]
\caption{Examples of sequential probabilistic inference $p(x_t|z_{\le t})$. Here $x_t$ is the abstract predictive variable, $z_t$ is the observation and $s_t$ is the physical state.}
\label{tab:framework}
\centering
\resizebox{\linewidth}{!}{
\begin{tabular}{llll}
\hline
Task             & Predictive Variable                      & Observation          & Inference Goal                                                 \\ \hline
Forecasting      & $x_t=s_{t+1:t+H}$                & $z_t=s_t$            & $p(s_{t+1:t+H}|s_{\le t})$                           \\
Planning   & $x_t=(s_{t+1:t+H}, a_{t:t+H-1})$ & $z_t=(s_t, a_{t-1})$ & $p(s_{t+1:t+H}, a_{t:t+H-1}|s_{\le t}, a_{\le t-1})$ \\ 
State Estimation & $x_t=s_t$                        & $z_t\sim p(z_t|s_t)$ & $p(s_t|z_{\le t})$                                   \\\hline
\end{tabular}
}
\vspace{-0.6em}
\end{table*}
The idea of flow matching is to learn an ordinary differential equation (ODE) $\frac{d}{d\tau}x(\tau) = v(x(\tau), \tau)$ ($0\le \tau\le 1$) that transports between the source distribution $x(1)\sim p_1$ and the target distribution $x(0)\sim p_0$~\citep{lipman2023flow, liu2023flow}. To avoid confusion with physical time $t$, throughout the paper we use the symbol $\tau$ to denote the virtual time variable (or noise level) in flow matching. To construct such velocity field $v(x(\tau),\tau)$, we define an interpolation path $x(\tau)=\alpha(\tau) x(0)+\sigma(\tau) x(1)$, where $\alpha(\tau),\sigma(\tau)$ are coefficients satisfying end-point constraints $\alpha(0)=\sigma(1)=1$ and $ \alpha(1)=\sigma(0)=0$. By defining conditional velocity $\dot{x}(\tau) = \dot{\alpha}(\tau)x(0)+\dot{\sigma}(\tau)x(1)$, it is shown that the velocity $v(x(\tau),\tau)$ can be expressed by $v(x(\tau),\tau)=\mathbb{E}(\dot{x}(\tau)|x(\tau))$, and can be learned by a neural network $v_{\theta}(x(\tau),\tau)$ through optimization: $
    \min_{\theta} \mathbb{E}_{(x(0),x(1))\sim \pi, \tau\sim  \text{Uniform}(0,1)}\norm{v_{\theta}(x(\tau),\tau)-\dot{x}(\tau)}^2,$
where $\pi$ is a coupling of $p_0,p_1$, i.e., $\pi$ represents a joint distribution of $(x(0),x(1))$ whose marginals are $p_0,p_1$. 
To sample from target distribution $p_0$, we solve the ODE $dx(\tau)/d\tau=v_{\theta}(x(\tau),\tau)$ starting from source distribution $x(1)\sim p_1$ that is easy to sample from. The probability flow ODE counterpart of Diffusion models can be viewed as a special case of flow matching with $p_1=\mathcal{N}(0, I)$~\citep{song2021scorebased}. 
\section{Sequential Bayesian Flow Matching}
\subsection{Problem Setup and Bayesian Filtering}
\label{sec:bayesian_filtering}

 We study a sequential probabilistic inference problem, i.e., modeling the evolving posterior distributions with new observations arriving over time. Formally, let $x_t$ be a random variable to be predicted at the physical time $t$, and let $z_t$ be the observation received at time $t$. Given streaming observations $z_{\le t}:=(z_1,z_2,...,z_t)$, the goal is to infer the posterior distribution $p(x_t|z_{\le t})$, over $t=1,2,...,T$.

Many streaming tasks can be viewed as special instances of this formulation. A particularly interesting case we consider is when $x_t$ represents a \emph{long-horizon future trajectory}. For example, let $s_t$ denote the real system state at time $t$. In online forecasting, the task is to infer $H$-step future states $x_t=s_{t+1:t+H}$ with historical states $z_{\leq t}$ where $z_i =s_{i}$. In planning with generative models~\citep{janner2022planning}, the goal is to infer both future states and the actions to take $x_t = (s_{t+1:t+H},a_{t:t+H-1})$ conditioned on historical states and actions $z_{\leq t}$ where $z_{i}=(s_i, a_{i-1})$. Table~\ref{tab:framework} summarizes these representative examples.


\textbf{Bayesian filtering.} Bayesian filtering provides a principled framework to recursively update posterior distribution $p(x_t|z_{\le t})$ using new accumulated observations $z_{\le t}$. Concretely, by Bayesian rule, the posterior $p(x_t|z_{\le t})$ satisfies the key recursive formula (\emph{Note that the process may not be Markovian}):
\begin{align}
    &\boldsymbol{p(x_t|z_{\le t})}\propto p(z_t|x_{t}, z_{\le t-1})p(x_t|z_{\le t-1}), \label{eq:bayesian_filtering}\\
    &p(x_t|z_{\le t-1})\!=\!\int \!p(x_t|x_{t-1},z_{\le t-1})\boldsymbol{p(x_{t-1}|z_{\le t-1})}dx_{t-1}.  \label{eq:bayesian_filtering2}
\end{align}
Equations~\eqref{eq:bayesian_filtering},\eqref{eq:bayesian_filtering2} highlight a key structural property: the current posterior $p(x_{t}|z_{\le t})$ can be obtained by propagating the previous posterior $p(x_{t-1}|z_{\le t-1})$ via the dynamics $p(x_t|x_{t-1}, z_{\le t-1})$ and a correction using new observation $z_t$. This exact formulation is often intractable due to unknown system dynamics. Next, we abstract this process as a distributional-level operator:
\begin{equation}
\label{eq:bayesian_filtering_operator}
    p(x_t|z_{\le t})=\text{Bayesian\_filtering}(p(x_{t-1}|z_{\le t-1}); z_{\le t}).
\end{equation}

In practice, the posterior distribution $p(x_t|z_{\le t})$ is usually represented by a set of particles (samples) $\{x_t^{(i)}\}_{i=1,2,...,N}$. Particle filtering provides a concrete realization of the recursive update in Eq.~\eqref{eq:bayesian_filtering_operator} by evolving this set of particles over time. For example, particles are propagated according to system dynamics and adjusted based on new observations through weighting and resampling~\citep{gordon1993novel, smith2013sequential}.

\begin{figure*}[t!]
    \centering
    \includegraphics[width=1\linewidth]{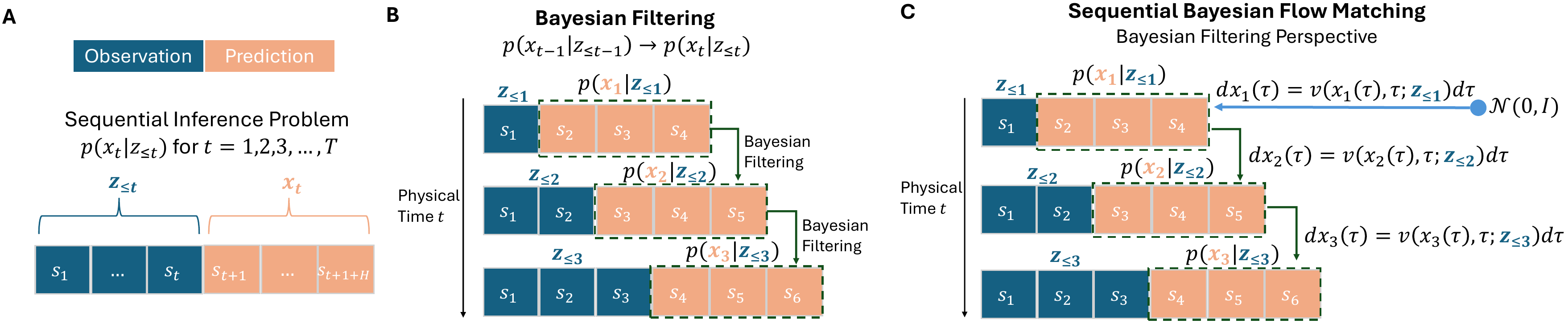}
    \caption{ \textbf{A}: Streaming forecasting in sequential probabilistic inference framework. 
    \textbf{B}: Bayesian filtering framework refines previous particles (samples) based on latest observations to obtain new particles. 
    \textbf{C}: Following the idea of Bayesian filtering,
    \model leverages a probability flow to recursively transport from previous belief $p(x_{t-1}|z_{\le t-1})$ to current belief $p(x_t|z_{\le t})$.}
    \label{fig:bf}
    \vspace{-6pt}
\end{figure*}
\subsection{Sequential Inference with Flow-based Models: Existing Practices}
\label{sec:existing}
A common approach to posterior inference of $p(x_t|z_{\le t})$ is to treat each time step as a conditional generation problem and model $p(x_t| z_{\le t})$ with a flow-based generative model. Specifically, the model transports samples from a fixed base distribution (e.g., $\mathcal{N}(0, I)$) to the target distribution $p(x_t|z_{\le t})$ conditioned on $z_{\le t}$. Concretely, we consider the following ODE: 
\begin{equation}
\label{eq:ode_pretrain}
\frac{d}{d\tau}x_t(\tau)
= v(x_t(\tau), \tau; z_{\le t}), 
\qquad
x_t(0)\sim p(x_t|z_{\le t}), \quad
x_t(1)\sim \mathcal{N}(0,I).
\end{equation}
where $\tau \in [0,1]$ denotes the flow-matching time and $t$ denotes the physical time step. Learning the velocity field and solving the ODE allow us to sample from $p(x_t \mid z_{\le t})$.

To deploy such a model for sequential inference, one typically applies the conditional generative model repeatedly over physical time $t$. This receding-horizon paradigm is widely used in time-series forecasting~\citep{box2015time} and decision-making systems~\citep{garcia1989model}. For flow-based models, however, this naive strategy requires restarting sampling and solving the ODE in Eq.~\eqref{eq:ode_pretrain} from scratch whenever a new observation arrives, often using tens or more solver steps~\citep{janner2022planning}. This incurs substantial latency and may fail to keep pace with the observation rate in real-time systems. This issue is further amplified when multiple particles are needed to represent the distribution over future trajectories. 

\subsection{Sequential Bayesian Flow Matching} 
\label{sec:sbfm}



Motivated by the recursive structure of Bayesian filtering, we propose \emph{Sequential Bayesian Flow Matching} (\model), a simple framework that directly parameterizes a data-driven analogue of the filtering process in Eq.~\eqref{eq:bayesian_filtering_operator} as a probability flow. The key idea is to replace the non-informative Gaussian source with the previous posterior $p(x_{t-1}|z_{\le t-1})$ in the flow ODE:
\begin{equation}
\label{eq:ode_sfm}
    \frac{d}{d\tau}x_t(\tau) = v(x_t(\tau), \tau;z_{\le t}),\qquad
    x_{t}(0)\sim p(x_t|z_{\le t}), \quad \boldsymbol{x_t(1)\sim p(x_{t-1}|z_{\le t-1})}.
\end{equation} 
By solving this ODE from $\tau=1$ to $0$, the flow model transports from the previous posterior to the current one. Figure~\ref{fig:bf} C illustrates the idea of \model.

\textbf{Informative prior and warm-start.} A key advantage of this approach is that, the previous posterior already encodes substantial information about the
next posterior in temporally coherent physical
system. It therefore provides an more informative source distribution
than an independent Gaussian. This warm-start reduces the gap that the flow needs to traverse, leading to smoother transport dynamics that are easier to learn and more \emph{efficient inference} with fewer flow sampling steps~\citep{ren2025prior,scholz2025warm, park2024leveraging}.

\textbf{Simple but not merely heuristic.} While \model is simple in form, it differs from existing warm-start methods by directly learning the posterior update rather than applying an ad hoc reuse strategy. If the learned transport accurately matches the true conditional posterior update, it recovers the desired recursive belief updates. Existing strategies typically adopt inference-time heuristics without an explicit objective for posterior-update modeling. One common approach perturbs previous samples and denoises them under new observations using the original diffusion model~\citep{janner2022planning, duan2025real, li2026step}, but that model is still trained for Gaussian-to-target generation, not posterior-to-posterior transport. 
Another line of work trains asynchronous denoising diffusion models, which maintain future trajectories at different noise levels along the prediction horizon and progressively denoise them as new observations arrive~\citep{wei2025cldiffphycon, hoeg2025fast}. However, because future states are intentionally kept partially denoised, the resulting estimates can remain noisy. In contrast, \model learns the posterior-to-posterior flow induced by new observations, which also leads to a different particle-based mechanism for distribution modeling.

\textbf{Distribution modeling.} Modeling the full posterior distribution, rather than a
single point estimate, is crucial for uncertainty quantification and risk-sensitive downstream forecasting and decision making~\cite{ghahramani2015probabilistic, gneiting2014probabilistic, mesbah2016stochastic}. Standard diffusion models draw $N$ particles (samples) from Gaussian noise and denoise under conditions $z_{\le t}$, resulting in a particle set $\{x_t^{(i)}\}_{i=1,2,...,N}$ representing $p(x_t|z_{\le t})$~\citep{rasul2021autoregressive, yan2021scoregrad}. This procedure regenerates the particle set from
scratch after every new observation.

\model instead maintains and evolves a set of particles $\{x_t^{(i)}\}_{i=1,2,...,N}$ over time. When a new observation $z_{t+1}$ arrives, each previous particle $x_t^{(i)}$ is evolved according to the flow ODE Eq.~\eqref{eq:ode_sfm} conditioned on $z_{\le t+1}$, resulting in an updated particle set $\{x^{(i)}_{t+1}\}_{i=1,2,...,N}$ that approximates new posterior $p(x_{t+1}|z_{\le t+1})$. Figure \ref{fig:particle_traj} illustrates the particle evolution of \model on a synthetic dataset. Particles are moved away from low-probability regions toward regions consistent with the new observations, and old modes inconsistent with new observations are progressively filtered out.

\textbf{Connection to particle filtering.} This particle-evolution view is closely connected to particle-based Bayesian filtering. The classic particle filtering, or Sequential Monte Carlo~\cite{gordon1993novel, smith2013sequential}, also maintains a set of particles to approximate the posterior. The particles are propagated via system dynamics, re-weighted according to new observations and then re-sampled. The importance weighting is known to suffer from particle degeneracy in high-dimensional settings, where the posterior mass concentrates
on only a few particles and most particles receiving negligible weights are discarded~\citep{smith2013sequential}. Particle flow filtering is later developed to mitigate this issue by directly constructing a probability flow that bridges between posteriors~\citep{daum2010exact, daum2011particle, khan2014non}. Particles are allowed to move to high-density regions rather than being discarded, similar to the mechanism of \model shown in Figure~\ref{fig:particle_traj}. These classic approaches derive flow dynamics from explicit system and observation models, while our method learns the velocity field directly from data without requiring knowledge of the underlying system.

\begin{figure}[t]
\vspace{-15pt}
  \centering
  \begin{subfigure}[t]{0.57\columnwidth}
    \centering\includegraphics[width=\linewidth]{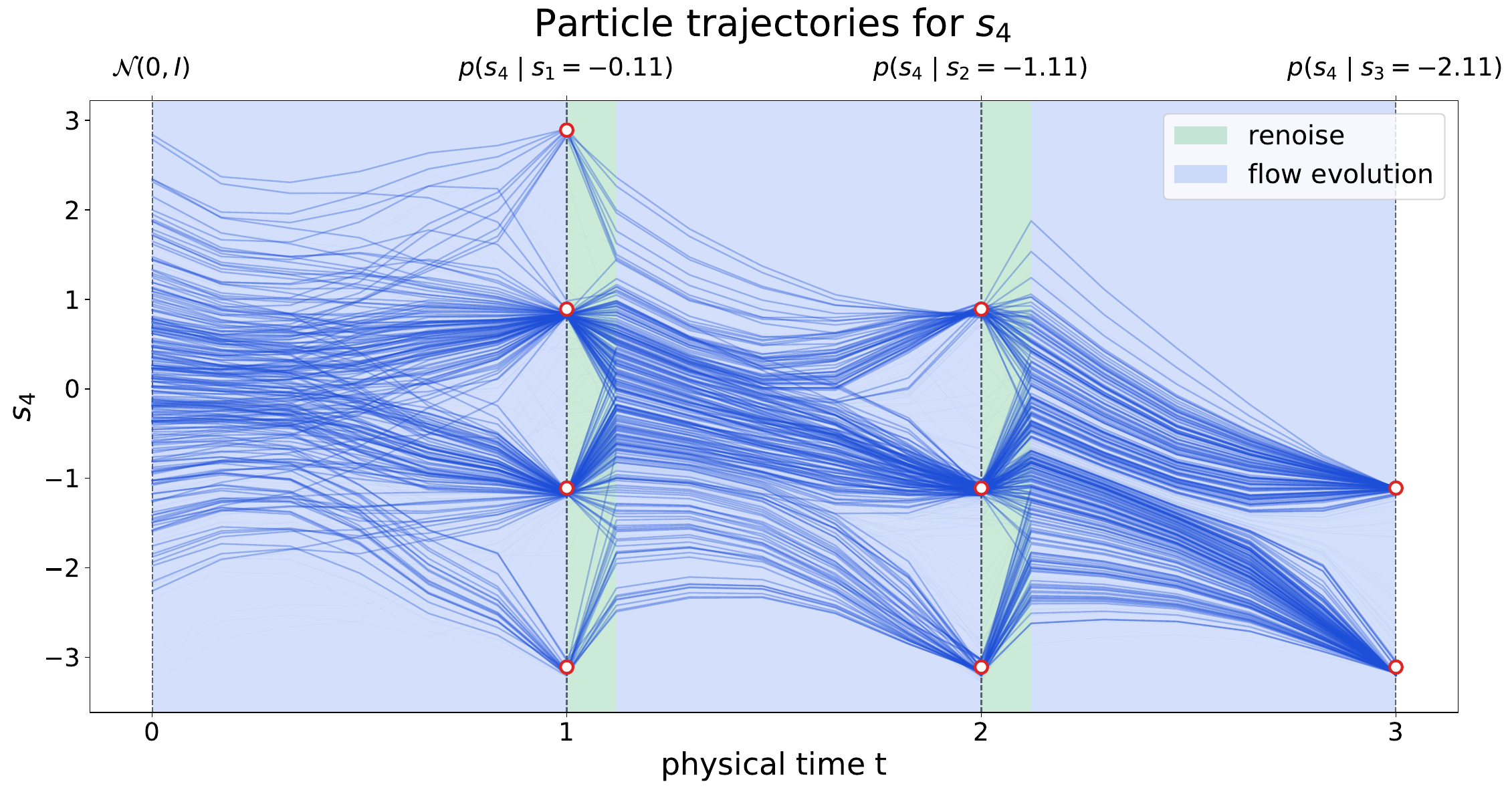}
    \caption{Particle evolution of $p(s_4|s_t)$ by \model
    }
    \label{fig:particle_traj}
  \end{subfigure}
  \hfill
  \begin{subfigure}[t]{0.42\columnwidth}
    \centering
\includegraphics[width=\linewidth]{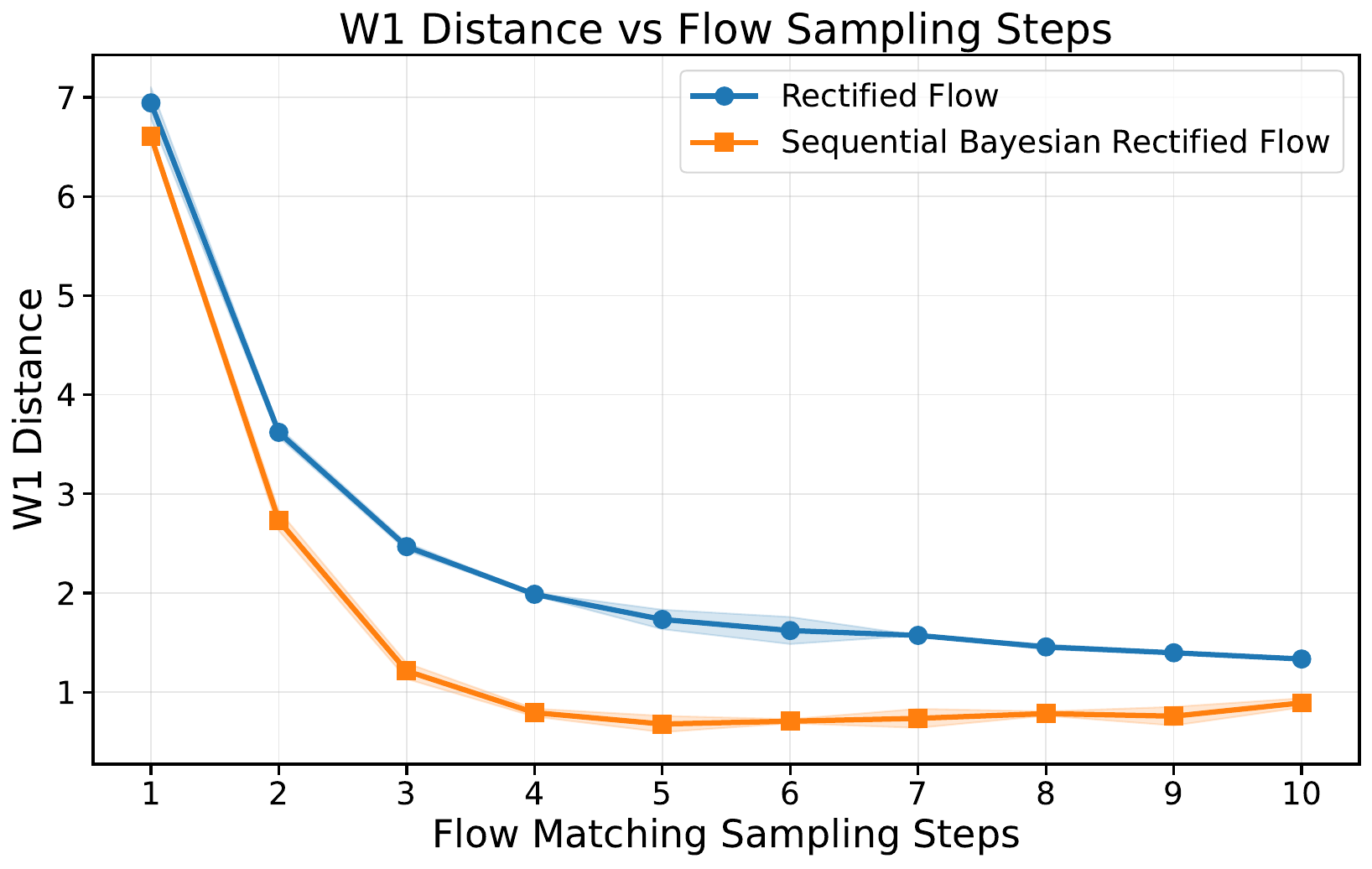}
    \caption{Quantitative comparison}
  \label{fig:w1_dist}
  \end{subfigure}
  \caption{Forecasting on a Bernoulli random walk
$p(s_{t+1}|s_t)=\tfrac{1}{2}\delta_{s_t+1}+\tfrac{1}{2}\delta_{s_t-1}$,
with $s_1\sim\mathcal{N}(0,10^{-4})$. The task is to infer the posterior trajectory distribution
$p(s_{t+1:t+5}|s_t)$ for $t=1,\ldots,4$. Each marginal $p(s_{t+\Delta}|s_t)$ has $(\Delta+1)$ discrete modes.
(a) Evolution of 500 \model particles for the marginal state $s_4$. As each new observation $s_t$ arrives, \model transports particles from $p(s_4| s_{t-1})$ to $p(s_4|s_t)$ following the flow ODE Eq.~\eqref{eq:ode_sfm}. Red circles mark ground-truth modes. Inconsistent modes are progressively filtered out.
(b) Average 1-Wasserstein distance between generated and ground-truth distributions versus flow sampling steps. \model outperforms the full-step standard flow model, using only three flow sampling steps. Training details are in Appendix~\ref{appx:exp}.}
  \label{fig:syn}
  \vspace{-6pt}
\end{figure}

\subsection{Training and Inference Algorithms}
\label{sec:sfm_algo}

We discuss how to learn \model  (Eq.~\eqref{eq:ode_sfm}) via flow matching. The flow matching objective requires pairs of samples $(x_{t-1},x_t)$ as well as condition $z_{\le t}$. Suppose we have an offline dataset of state-observation trajectories $\{(x^{*, m}_{1:T}, z^{*, m}_{1:T})\}_{m=1,2,...,M}$. For notational simplicity, we omit the trajectory index,
denoting a general ground-truth trajectory from the dataset as $(x^*_{1:T}, z^*_{1:T})$.


\textbf{Training Data Construction.}
A na\"ive approach is to draw successive ground-truth states $(x^*_{t-1}, x^*_t)$ together with observations $z^*_{\le t}$ from the offline dataset, and train a conditional flow matching to transport $x^*_{t-1}$ to $x^*_t$ conditioned on $z^*_{\le t}$. However, this construction introduces a fundamental distributional mismatch. Since the data $(x_{t-1}^*, x_t^*, z_{\le t}^*)$ are drawn from the joint distribution $p(x_{t-1}, x_t, z_{\le t})$, the source sample $x_{t-1}^*$ exposed to the conditional flow model $v_{\theta}(\cdot|z^*_{\le t})$ has marginal distribution $p(x_{t-1}|z^*_{\le t})$. However, Bayesian filtering recursion requires the previous posterior as $p(x_{t-1}|z^*_{\le t-1})$, without observing $z_t^*$.  As a result, the learned flow matching is not aligned with the posterior-to-posterior flow ODE Eq.~\eqref{eq:ode_sfm}. A formal analysis of this mismatch is in Appendix~\ref{appx:proof} and this training approach empirically leads to a degraded performance shown in Appendix~\ref{appx:ablation}.



To address this mismatch, we use a simple pretraining-finetuning framework that
adapts a pretrained diffusion or flow model to learn the posterior-to-posterior
transport. In the pretraining stage, we train a standard flow-based model $v_{\mathrm{pre}}$ that transports Gaussian noise to $x_t^*$
conditioned on $z_{\le t}^*$ (Eq.~\eqref{eq:ode_pretrain}). In the finetuning
stage, we freeze the pretrained model and use it to generate source samples that
match the filtering recursion. Specifically, we (a) draw a ground-truth
state-observation pair $(x_t^*, z_{\le t}^*)$, and (b) sample
$x_{t-1}^{\mathrm{pre}} \sim
p_{\mathrm{pre}}(x_{t-1}|z_{\le t-1}^*)$ from the pretrained model
conditioned only on the previous observations. We then initialize
\model $v_\theta$ with pretrained model weights and finetune it to transport
$x_{t-1}^{\mathrm{pre}}$ to $x_t^*$ conditioned on new observations
$z_{\le t}^*$. This makes the source samples
approximately follow the previous posterior, so
the finetuned model learns the desired posterior update.



\textbf{Re-noising mechanism. } 
As the pretrained model only approximates the true distribution, discrepancies between the training and test-time source distributions can lead to error accumulation over successive belief updates. To mitigate the errors, we introduce a re-noising mechanism that injects controlled noise into the previous estimate $x_{t-1}$ during both training and inference. Concretely, in filtering flow ODE Eq.~\eqref{eq:ode_sfm}, we replace the source distribution $p(x_{t-1}|z_{\le t-1})$ with a noisy version $\tilde{x}_{t-1} := \alpha(\tau_{\text{renoise}}) x_{t-1} + \sigma(\tau_{\text{renoise}})\cdot \mathcal{N}(0, I)$, where $\alpha(\tau_{\text{renoise}})$ and $\sigma(\tau_{\text{renoise}})$ are the coefficients of the probability flow path at flow time $\tau_{\text{renoise}}$. We treat $\tau_{\text{renoise}}$ as a hyperparameter that can be tuned.

The re-noise level $\tau_{\text{renoise}}$ controls the trade-offs between preserving historical posterior's information for sampling acceleration and accommodating uncertainty during belief refinement. If $\tau_{\text{renoise}}=1$, we have $\tilde{x}_{t-1}\sim \mathcal{N}(0,I)$, which reduces to the restarting-sampling paradigm that requires a large number of flow sampling steps. If $\tau_{\text{renoise}}=0$, it relies on a precise modeling of $x_{t-1}\sim p(x_{t-1}|z_{\le t-1})$, leading to risk of error accumulation. Appendix~\ref{appx:ablation} studies the impact of $\tau_{\text{renoise}}$.


\textbf{Flexible particle set size.}
The number of particles $N$ controls the granularity of the posterior representation and can be adjusted according to the latency budget in practice. For example, the particle set can be expanded by branching existing particles via independent re-noising ($x_t^{(i)} \rightarrow \{\tilde{x}_t^{(i,1)}, \tilde{x}_t^{(i,2)},\ldots\}$), or by drawing fresh particles from the pretrained model when additional latency is acceptable. In our experiments, we keep $N$ fixed and use \model for all sequential updates. 

Algorithms~\ref{alg:sfm_train} and~\ref{alg:sfm_inference} present the finetuning and inference pipelines. At inference, the initial particles $x_1^{(i)}$ are generated by the pretrained flow model since no previous particles are available.

\begin{algorithm}[t]
\caption{Sequential Bayesian Flow Matching Finetuning}
\label{alg:sfm_train}
\begin{algorithmic}[1]
\Require Ground-truth $x_t$, context $z_{\le t}$, renoise level $\tau_{\text{renoise}}$, and a pretrained flow model ${v_{\text{pre}}}$ 
    \State Generate $x^{\text{pre}}_{t-1}\sim p_{{\text{pre}}}(x_{t-1}|z_{\le t-1})$ using $v_{\mathrm{pre}}$ (solve ODE Eq.~\eqref{eq:ode_pretrain})
    \State Renoise $\tilde{x}_{t-1}= \alpha(\tau_{\text{renoise}})\cdot x_{t-1}^{\text{pre}}+\sigma(\tau_{\text{renoise}})\cdot \mathcal{N}(0,I)$
    \State Sample interpolation time $\tau\sim \text{Uniform}(0,1)$
    \State Compute $x_t(\tau)=\alpha(\tau)x_t+\sigma(\tau)\tilde{x}_{t-1}$ 
\State \Return loss $L(\theta)=\norm{v_{\theta}(x_t(\tau), \tau;z_{\le t}) - \dot{x}_t(\tau)}^2$
\end{algorithmic}
\end{algorithm}

\begin{algorithm}[t]
\caption{Sequential Bayesian Flow Matching Inference}
\label{alg:sfm_inference}
\begin{algorithmic}[1]
\Require Pretrained flow model $v_{\text{pre}}$, \model $v_\theta$ finetuned by Algorithm~\ref{alg:sfm_train}, initial observation $z_1$, particle size $N$
\State Initialize particle set $\{x_1^{(i)}\}_{i=1,2,...,N}$ with    $x^{(i)}_1\overset{\scriptstyle \text{i.i.d.}}{\sim} p_{{\text{pre}}}(x_1|z_1)$ (solve ODE Eq.~\eqref{eq:ode_pretrain})
\For{$t = 2$ to $T$}
     \State Receive new observation $z_t$
     \For{$i=1$ to $N$}{\hspace{5pt}(run in parallel)}
     \State $\tilde{x}^{(i)}_{t-1}= \alpha(\tau_{\text{renoise}})\cdot x^{(i)}_{t-1}+\sigma(\tau_{\text{renoise}})\cdot \mathcal{N}(0, I)$
     \State Solve $dx^{(i)}_t(\tau)/d\tau=v_{\theta}(x^{(i)}_t(\tau),\tau;z_{\le t})$ from $x^{(i)}_t(1)=\tilde{x}^{(i)}_{t-1}$, and obtain $x^{(i)}_t(0)$
     \EndFor
     \State Obtain new particle set $\{x_t^{(i)}\}_{i=1,2,...,N}$ by assigning $x_t^{(i)}=x_t^{(i)}(0)$
\EndFor
\end{algorithmic}
\end{algorithm}


\section{Related Works}
\label{sec:related}
\textbf{Flow-based models for trajectory generation.}
Diffusion and flow matching models have been widely used for trajectory generation in an offline setting.
Existing approaches mainly differ in how denoising is scheduled across the prediction horizon, including full-sequence denoising~\citep{li2022diffusion, ho2022video}, autoregressive denoising~\citep{hoogeboom2021autoregressive, rasul2021autoregressive} and asynchronous denoising~\citep{pmlr-v235-ruhe24a, wu2023ar, chen2024diffusion}. Asynchronous denoising allows assigning and denoising variables of different noise levels along the prediction axis.
These methods are primarily designed for static trajectory generation and do not account for streaming observations. When applied in streaming settings, they rely on repeated sampling from scratch.

\textbf{Efficient sampling of flow-based models.}
Prior work accelerates sampling by applying advanced numerical solvers~\citep{lu2022dpm, zhangfast} or distillation-based methods that compress the sampling path of a teacher model into a fewer-step student model~\citep{salimans2022progressive, yin2024one, geng2023one}.
More recently, flow map matching methods directly learn the solution operator of the flow ODE and enable few-step generation~\citep{song2023consistency, boffi2025flow, geng2025mean}. However, the training objective is known to be unstable to optimize~\citep{lusimplifying, geng2025improved}. In our streaming experiments, consistency models~\citep{geng2025consistency} and MeanFlow~\citep{geng2025mean} were difficult to optimize reliably and generalized poorly. Extending such one-step generative models to streaming tasks is a nontrivial direction, and could potentially be combined with our framework in future work.

\textbf{Bayesian inference of diffusion models.} Diffusion priors have been studied for Bayesian inverse problems, where the goal is to infer $p(x|y)$ from measurements $y=f(x)$ and a diffusion prior $p(x)$~\citep{cardoso2024monte, chung2023diffusion, dou2024diffusion, daras2024survey}. These approaches mainly focus on static inference with single-shot measurements. While recent work has begun to explore flow-based data assimilation~\citep{chen2025flowdas}, sequential probabilistic inference with diffusion or flow priors remains underexplored, especially in low-latency streaming settings. Crucially, most approaches rely on likelihood-based guidance, requiring explicit, often linear, observation models and system dynamics that are unavailable in the real world.

\section{Experiment}
\label{sec:exp}

\begin{figure}[t]
\vspace{-16pt}
\centering\begin{minipage}[t]{0.49\linewidth}
\vspace{0pt}
    \centering
    \captionof{table}{Results of beam spill forecasting. Prediction horizon $H=300$.}
    \label{tab:spill}
    \resizebox{\linewidth}{!}{
    \begin{tabular}{lccc}
        \toprule
        \textbf{Method} & NFE & Energy Score $\downarrow$ & RMSE $\downarrow$ \\
        \midrule
        Autoregressive & $H$ & $1.2509_{\pm 0.0013}$ & $2.199_{\pm 0.021}$  \\
        Diffusion & 5 & $0.0196_{\pm 0.0001}$ & $\textbf{0.051}_{\pm 0.002}$ \\
        Rectified Flow & 5 & $0.0231_{\pm 0.0003}$ & $0.072_{\pm 0.004}$ \\
        Diffusion & 1 & $0.0220_{\pm 0.0001}$  & $0.062_{\pm 0.002}$ \\
        Rectified Flow & 1 & $0.0302_{\pm 0.0007}$  & $0.115_{\pm 0.009}$  \\
        MeanFlow & 1 & $0.1780_{\pm 0.0022}$ & $0.849_{\pm 0.280}$  \\
        Consistency Model & 1 & $0.0269_{\pm 0.0007}$ & $0.155_{\pm 0.044}$  \\
        \midrule
        Asynchronous Diffusion & 1 & $0.0336_{\pm 0.0009}$ & $0.358_{\pm 0.011}$  \\
        Warm-start Diffusion & 1 & $0.0197_{\pm 0.0002}$ &  $0.053_{\pm 0.002}$ \\
        Warm-start Diffusion & 2 & $0.0203_{\pm 0.0001}$ &  $0.054_{\pm 0.001}$ \\
        Warm-start Rectified Flow & 1 & $0.0236_{\pm 0.0001}$  & $0.076_{\pm 0.001}$ \\
        Warm-start Rectified Flow & 2 & $0.0239_{\pm 0.0002}$  & $0.083_{\pm 0.001}$ \\
        \midrule
        \textbf{\model-Rectified Flow} & 1 & $0.0179_{\pm 0.0003}$ & $0.054_{\pm 0.001}$ \\
        \textbf{\model-Rectified Flow} & 2 & $\textbf{0.0170}_{\pm 0.0002}$ & $\textbf{0.050}_{\pm 0.002}$ \\
        \bottomrule
    \end{tabular}
    }
\end{minipage}
\hfill
\begin{minipage}[t]{0.49\linewidth}
    \vspace{0pt}
    \centering
    \includegraphics[width=\linewidth]{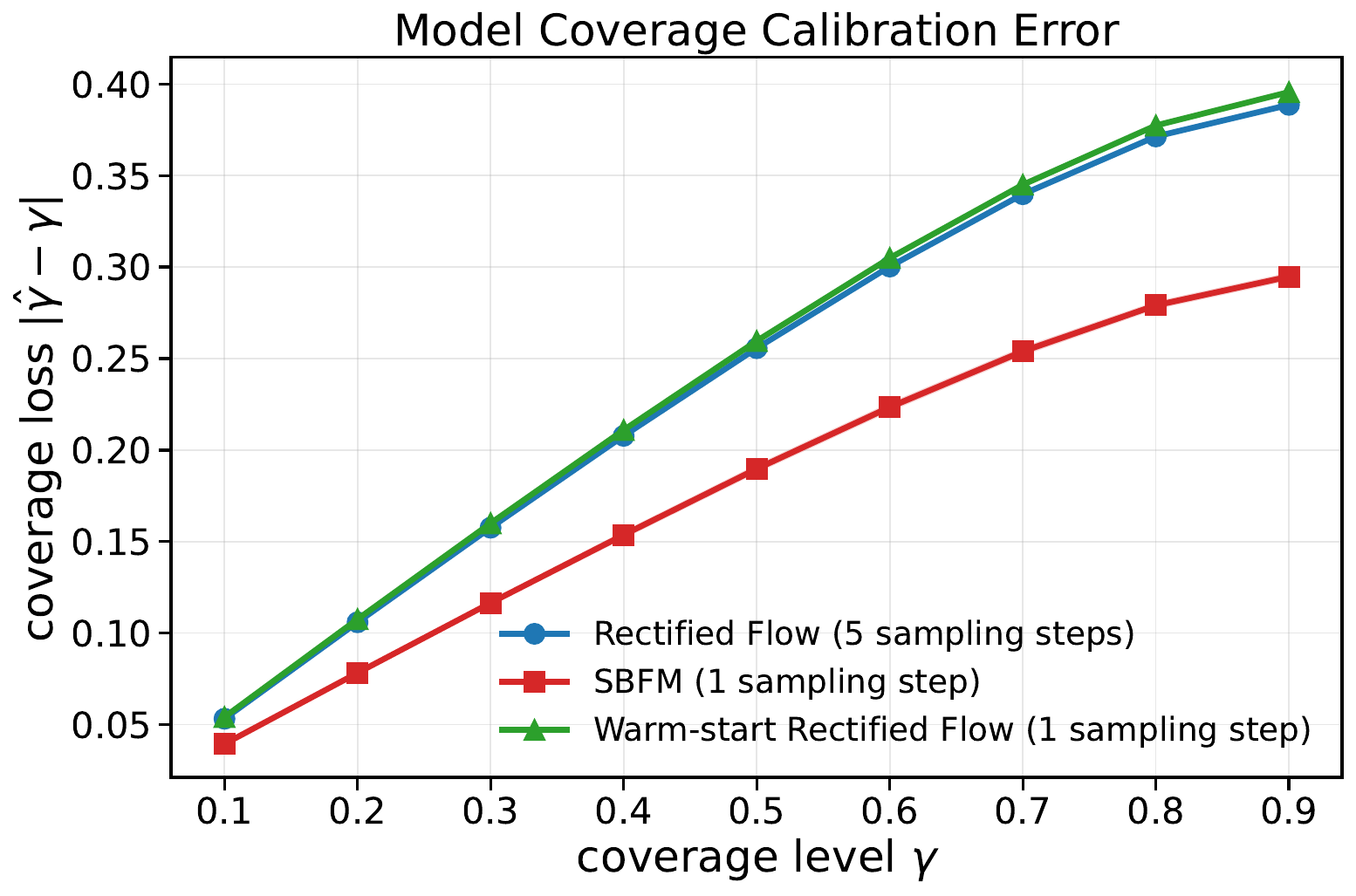}
    \captionof{figure}{Predictive coverage error on beam-spill forecasting.}
    \label{fig:spill}
\end{minipage}
\vspace{-4pt}
\end{figure}

We evaluate \model primarily on scientific forecasting tasks that require inferring the future evolution of stochastic dynamical systems from streaming observations. 
These include \emph{beam spill dynamics} in a particle accelerator system~\citep{whitbeck2025fast}, as well as fluid dynamics and weather forecasting~\citep{moro2025solving, rasp2024weatherbench}. 
We show that \model preserves distributional accuracy while greatly reducing inference latency. 
Additionally, we consider planning and control tasks~\citep{fu2020d4rl, Holl2020Learning}, where \model enables faster closed-loop re-planning. 
Details of the datasets and model implementation are deferred to Appendix~\ref{appx:exp}.



\textbf{Pretrained Model.} We mainly adopt Diffusion with DDIM sampler~\citep{song2021denoising} and Rectified flow~\citep{liu2023flow} as the pretrained model. 
Note that \model is compatible with general flow-based models, though we evaluate it using these two representative models for demonstration.

\textbf{Baselines.} We consider several classes of efficient flow-based baselines:
(1) \emph{One-step diffusion}: We evaluate consistency models \citep{song2023consistency} and MeanFlow \citep{geng2025mean}. For consistency models, we adopt ECT \citep{geng2025consistency}, an efficient fine-tuning method for consistency distillation.
(2) \emph{Warm-start diffusion}: Following~\citep{janner2022planning, duan2025real}, we add a controlled amount of noise to the previous prediction $x_{t-1}$ and denoise it under the updated condition $z_t$ using the same pretrained model. For direct comparison, we set the noise level to $\tau_{\text{renoise}}$, matching \model's re-noise level. This warm-start heuristic can be viewed as an ablation of \model, where fine-tuning stage is removed and the pretrained model is directly applied.
(3) \emph{Asynchronous denoising diffusion}: We adopt the idea from \citep{wei2025cldiffphycon, hoeg2025fast}, training an asynchronous diffusion and maintaining a partially denoised future trajectories at inference time.

 \textbf{Metrics of inference latency.} We mainly report the Number of Function Evaluations (NFE) for measuring sampling efficiency, defined as the number of first-order Euler steps used to solve the flow ODE.  Experiments are conducted on a single NVIDIA RTX 6000 Ada Generation (48 GB memory).

\textbf{Long-horizon $x_t$.} In our experiments, $x_t$ typically denotes a long-horizon trajectory, e.g., $x_t=[s_{t+1},\ldots,s_{t+H}]$. We define $\mathrm{shift}(x_t)=[s_{t+2},\ldots,s_{t+H},s_{t+H}]$ and practically train \model to transport $p(\mathrm{shift}(x_{t-1})\mid z_{\le t-1})$ to $p(x_t\mid z_{\le t})$, aligning states at the same physical time.


\begin{table*}[t!]
\vspace{-12pt}
    \centering
    \caption{Results of fluid system and weather forecasting. Prediction horizon $H=10$ and $12$ for Burgers'equation and weather forecasting.}
    \label{tab:physics}
    \resizebox{1\textwidth}{!}{
    \begin{tabular}{lcccccc}
        \toprule
        \multirow{2}{*}{\textbf{Method}} &  \multicolumn{3}{c}{\textbf{Burgers' Equation}} &\multicolumn{3}{c}{\textbf{Weather Forecasting}} \\
        & NFE & Energy Score $\downarrow$ & RMSE $\downarrow$ & NFE & Temperature Energy Score $\downarrow$ & Temperature RMSE $\downarrow$ \\
        \midrule
        Autoregressive & $H$ & $0.133_{\pm 0.009}$ & $0.253_{\pm 0.022}$ & $H$ & $2.583_{\pm 0.026}$ & $3.764_{\pm 0.042}$ \\
        Diffusion & 10 & $\textbf{0.100}_{\pm 0.001}$ & $0.258_{\pm 0.005}$ & 20 & $2.173_{\pm 0.084}$ & $4.249_{\pm 0.162}$ \\
        Rectified Flow & 5 & $0.103_{\pm 0.001}$ & $0.262_{\pm 0.001}$ & 10 & $\textbf{2.030}_{\pm 0.080}$ & $3.990_{\pm 0.146}$  \\
        Diffusion & 1 & $0.108_{\pm 0.002}$ & $0.247_{\pm 0.006}$ & 1 & $2.849_{\pm 0.197}$ & $4.974_{\pm 0.279}$ \\
        Rectified Flow & 1 & $0.114_{\pm 0.002}$ & $0.250_{\pm 0.006}$ & 1 & $2.608_{\pm 0.157}$ & $4.509_{\pm 0.228}$ \\ 
        MeanFlow & 1 & $0.234_{\pm 0.047}$ & $0.418_{\pm 0.070}$ & 1 & $5.700_{\pm 0.017}$ & $8.632_{\pm 0.030}$ \\
        Consistency Model & 1 & $0.107_{\pm 0.003}$ & $0.252_{\pm 0.004}$ & 1 & $2.610_{\pm 0.033}$ & $6.269_{\pm 0.067}$\\
        \midrule
        Asynchronous Diffusion & 1 & $0.264_{\pm 0.015}$ & $0.559_{\pm 0.011}$ & 1 & $3.368_{\pm 0.124}$ & $6.248_{\pm 0.086}$\\
        Warm-start Diffusion & 1 & $0.231_{\pm 0.013}$ & $0.487_{\pm 0.033}$ & 1 & $4.064_{\pm 0.065}$ & $7.072_{\pm 0.014}$\\
        Warm-start Rectified Flow & 1 & $0.218_{\pm 0.009}$ & $0.475_{\pm 0.004}$ & 1 & $3.209_{\pm 0.003}$ & $6.236_{\pm 0.108}$\\
        \midrule 
        \textbf{\model-Rectified Flow} & 1 & $\textbf{0.101}_{\pm 0.002}$ & $\textbf{0.239}_{\pm 0.005}$ & 1 & $2.353_{\pm 0.032}$ & $\textbf{3.660}_{\pm 0.061}$ \\
        \bottomrule
    \end{tabular}
    }
    \vspace{-6pt}
\end{table*}

\subsection{Mu2e Beam Spill Monitoring}
\label{exp:beam}

The Mu2e experiment at Fermilab is a particle-physics experiment designed to search for new physics~\cite{whitbeck2025fast}. In its beam spill system, charged particles are extracted over time, producing a beam intensity that evolves stochastically. The beam intensity must remain highly uniform to preserve detector live time and prevent irreversible loss of valuable physics data~\citep{narayanan2022machine}. Since the spill regulation system operates at 10 kHz with an overall feedback loop on the order of 1 ms, low-latency monitoring and prediction of the evolving beam state are crucial.

\textbf{Datasets.} We adopt a physics-based beam spill simulator to simulate beam intensity trajectories $s_{1:430}$ regulated by a controller over a 43ms time window~\cite{narayanan2022machine}. The task to infer 300-step future trajectories $p(s_{t+1:t+300}|s_{\le t}, a_{\le t+300})$ given historical states $s_{\le t}$ and input control signals $a_{\le t+300}$. 

\textbf{Distribution evaluation.}
We use $N=20$ particles to represent the predictive distribution. To evaluate distributional quality, we adopt the Energy Score, a proper scoring rule widely used in probabilistic time-series forecasting~\citep{gneiting2007strictly}. Given a ground-truth trajectory $x_t=s_{t+1:t+300}$ and the model's predictive distribution $p_\theta$, the energy score is defined as $\mathbb{E}_{\hat{x}_t\sim p_\theta}\|\hat{x}_t-x_t\|_1
-
\frac{1}{2}
\mathbb{E}_{\hat{x}_t,\hat{x}'_t\sim p_\theta}
\|\hat{x}_t-\hat{x}'_t\|_1.$
The energy score measures the absolute error but with a penalty of over-deterministic prediction. The energy score is proper, meaning its expected value is minimized when the predictive distribution matches the ground-truth distribution~\citep{gneiting2007strictly}. As a complementary pointwise metric, we also report root mean squared error (RMSE), although RMSE alone does not directly assess distributional quality.

\textbf{Results.} Table~\ref{tab:spill} reports the energy score and RMSE of different methods. We observe that a deterministic auto-regressive model suffers from exposure bias, causing errors to blow up over the 300-step prediction horizon. Asynchronous diffusion intentionally maintain partially denoised future trajectories. For evaluation, we fully denoise these trajectories into clean samples using one-step DDIM. Despite this, their performance remains inferior. Warm-start heuristic performs reasonably on this task, but later experiments show that it does not generalize reliably across datasets. In contrast, \model finetuned from a pretrained rectified flow attains the best energy score with significantly less sampling steps than all other approaches.

Figure~\ref{fig:spill} presents a calibration case study on the beam-spill dataset. For each coverage level $\gamma\in(0,1)$, we construct the central $\gamma$ quantile interval from the particles and compute its empirical coverage $\hat{\gamma}$ across test trajectories and prediction horizons. We report the coverage loss $|\hat{\gamma}-\gamma|$. One-step \model consistently achieves lower coverage loss than the full-step pretrained model and warm-start heuristics, indicating better-calibrated predictive uncertainty. Together with the energy score results, this suggests that \model improves distributional quality with calibrated uncertainty.


\textbf{Performance-latency trade-offs.}
Appendix~\ref{appx:latency}, Figure~\ref{fig:spill_time} shows that across varying sampling steps, SBFM matches or outperforms the full-step pretrained model with fewer sampling steps. Similar trends appear on other datasets (Figures~\ref{fig:maze_time}).

\subsection{Fluid and Weather System Forecasting}
\textbf{Datasets.} We further consider forecasting on two physical systems: (a) fluid dynamic system governed by the 1D Burgers’ equation, adapted from~\citep{wei2025cldiffphycon}. (b) real-world weather forecasting adapted from WeatherBench2~\citep{rasp2024weatherbench}. The model takes historical states $z_{\le t}:=s_{\le t}$ as conditions and infers $H$-step future states $x_t:=s_{t+1:t+H}$ over $t$. The prediction horizon is $H=10$ for fluid system and $H=12$ for weather forecasting.

\textbf{Results.} We use $N=20$ particles and  Table~\ref{tab:physics} reports the energy score and RMSE. We observe acceleration heuristics such as asynchronous denoising and warm starts underperform, while \model still maintains a competitive energy score with only one flow sampling step.



\subsection{Planning and Control}
\textbf{Datasets.}
We evaluate \model on: (a) maze planning~\citep{fu2020d4rl}, where an agent navigates around obstacles to a goal, and (b) smoke control~\citep{Holl2020Learning, wei2025cldiffphycon}, where smoke is steered toward a target exit in a 2D incompressible fluid governed by the Navier–Stokes equations. For smoke control, we test both fixed maps and random maps with perturbed obstacle positions. The model conditions on $z_{\le t}:=(s_t,a_{t-1})$ and predicts $x_t:=(s_{t+1:t+H},a_{t:t+H-1})$. The planning horizon is $H=600,800,10$ for Maze-Medium, Maze-Large, and smoke control, respectively. It re-plans (incorporates new observations) every 50 steps for maze planning and every step for smoke control.

We use a variant of Diffusion tailored for sequence data named Diffusion Forcing as the pretrained model~\citep{chen2024diffusion}. We also include other classic baselines~\citep{williams2015model, kumar2020conservative, kostrikov2022offline, pomerleau1988alvinn, zhuang2023behavior}, whose reported numbers are from \cite{chen2024diffusion} and \cite{wei2025cldiffphycon} in the same benchmark settings.


\textbf{Results.} We use $N=1$ particle. Tables~\ref{tab:maze} and~\ref{tab:smoke} report results on maze planning and smoke control respectively. 
\model consistently achieves highly competitive downstream performance compared to full-step diffusion on both tasks with a small number of sampling steps.

\begin{table*}[t!]
\vspace{-12pt}
\centering
\begin{minipage}[t]{0.49\textwidth}
    \centering
    \caption{Total reward on maze planning task. Planning horizon is $H=600/800$ for Maze Medium/Large.}
    \label{tab:maze}
    
    \resizebox{1.0\textwidth}{!}{
    \begin{tabular}{lccc}
        \toprule
        \textbf{Method} & \textbf{NFE} & \textbf{Maze Medium} $\uparrow$ &  \textbf{Maze Large} $\uparrow$ \\
        \midrule
        MPPI~\citep{williams2015model} & N/A & $10.2$ & $5.1$ \\
        CQL~\citep{kumar2020conservative} & N/A & $5.0$ & $12.5$ \\
        IQL~\citep{kostrikov2022offline} & N/A & $34.9$ & $58.6$  \\
        Diffuser & 256 & $121.5_{\pm 2.7}$ & $123.0_{\pm 6.4}$  \\
        Diffusion Forcing  & 50 & $149.4_{\pm 7.5}$ & $159.0_{\pm 2.7}$  \\
        Diffusion Forcing  & 3 & $110.4_{\pm 11.6}$ & $67.3_{\pm 27.1}$ \\
        Diffusion Forcing  & 1 & $22.1_{\pm 12.3}$ & $2.2_{\pm 1.9}$ \\
        \midrule
        Asynchronous Diffusion & 1 & $74.2_{\pm 2.4}$ & $96.1_{\pm 3.6}$\\
        Asynchronous Diffusion & 3 & $78.3_{\pm 0.7}$ & $103.9_{\pm 4.1}$ \\
        Warm-start Diffusion Forcing & 1 & $85.1_{\pm 12.6}$ & $61.9_{\pm22.1}$\\
        Warm-start Diffusion Forcing & 3 & $81.8_{\pm 8.2}$ & $62.3_{\pm20.7}$\\
        \midrule  
        \textbf{\model-Diffusion Forcing} & 1 & $57.5_{\pm 29.3}$ & $114.6_{\pm 49.0}$\\
        \textbf{\model-Diffusion Forcing} & 3 & $\textbf{168.6}_{\pm 11.9}$ & $\textbf{233.9}_{\pm 8.8}$\\
        \bottomrule
    \end{tabular}
    }
\end{minipage}
\hfill
\begin{minipage}[t]{0.49\textwidth}
\vspace{0pt}
    \centering
    \caption{Off-target leakage rate on smoke control task. Planning horizon is $H=10$.}
        \label{tab:smoke}
    
    \resizebox{0.99\textwidth}{!}{
    \begin{tabular}{lccc}
        \toprule
        \multirow{1}{*}{\textbf{Method}} & \multirow{1}{*}{\textbf{NFE}} &  Fixed Map $\downarrow$ & Random Map $\downarrow$ \\
        \midrule
        BC~\citep{pomerleau1988alvinn} & N/A & $0.672$ & $0.705$ \\
        BPPO~\citep{zhuang2023behavior} & N/A & $0.634$ & $0.652$ \\
        DiffPhyCon & 600 & $0.545$ & $0.375$  \\
        Diffusion Forcing  & 10 & $0.166_{\pm 0.036}$  &$0.291_{\pm 0.042}$\\
        Rectified Flow & 10 & $0.173_{\pm 0.078}$ & $0.190_{\pm 0.020}$ \\ 
        Diffusion Forcing  & 1 & $0.362_{\pm 0.068}$ &$0.280_{\pm 0.019}$ \\
        Rectified Flow & 1 & $0.432_{\pm 0.009}$ & $0.260_{\pm 0.020}$ \\ 
        MeanFlow & 1 & $0.369_{\pm 0.117}$ & $0.462_{\pm 0.166}$\\
        Consistency Model & 1 &$0.295_{\pm 0.013}$ & $0.293_{\pm 0.002}$\\
        \midrule
        Asynchronous Diffusion & 60  & $0.337$ & $0.346$  \\
        Warm-start Diffusion Forcing & 1 & $0.547_{\pm 0.026}$ &$0.268_{\pm 0.016}$\\
        warm-start Rectified Flow & 1 & $0.291_{\pm 0.039}$ & $0.286_{\pm 0.017}$\\
        \midrule 
        \textbf{\model-Diffusion Forcing} & 1 & $\textbf{0.097}_{\pm 0.012}$  &$\textbf{0.231}_{\pm 0.006}$\\
        \bottomrule
    \end{tabular}
    }
\end{minipage}
\vspace{-6pt}
\end{table*}




\subsection{Ablation Study}
\label{sec:ablation}
\vspace{-4pt}
The warm-start flow models can be viewed as an ablation of \model, where fine-tuning stage is removed and the pretrained model is used directly. We also conduct an ablation study on the renoising mechanism and the choice of training \model from raw data versus pretraining-finetuning strategy. Results are provided in Appendix~\ref{appx:ablation}.

\section{Conclusion and Limitations}
\label{sec:conclusion}
\vspace{-3pt}
This work studies the efficient deployment of flow-based models in sequential probabilistic inference, where models are required to infer the distribution of long-horizon future trajectory and continuously assimilate streaming observations under latency constraints. We propose Sequential Bayesian Flow Matching, which transports the predictive distribution across time steps and can be viewed as parameterizing the Bayesian belief update. Experiments show that it achieves distributional metrics competitive to full-step diffusion with significantly fewer sampling steps. 

A limitation of our approach is that recursive posterior-to-posterior updates may accumulate errors over time. A promising future direction is to develop adaptive re-noising strategies that regulate this accumulation, or to design criteria for when the model should reset by sampling from the Gaussian to recover from compounding errors.

\begin{ack}
We would like to thank Prof. Tailin Wu, as well as Ruiqi Feng, Long Wei and Haodong Feng for their insightful discussions and feedback on earlier versions of this work.

This work is primarily supported by NSF awards PHY-2117997, IIS-2239565, and IIS-2428777,
as well as Meta Research Award and Nvidia Academic Award. Prof. Bo Dai would like
to acknowledge support from NSF ECCS-2401391, NSF IIS-2403240, and ONR N000142512173.

\end{ack}

\makeatletter
\setlength{\bibhang}{0pt}
\renewcommand\@biblabel[1]{[#1]\hfill}
\makeatother
\bibliography{bibfile}
\bibliographystyle{unsrt}

\clearpage
\appendix
\section{Training \model From Raw Data}
\label{appx:proof}
Here we formally justify why training a conditional flow matching from successive ground-truth states $(x_{t-1}, x_{t})$ with conditions $z_{\le t}$ cannot implement filtering ODE Eq.~\eqref{eq:ode_sfm}, due to a mismatch of source distribution.

\textbf{Goal.} The goal is to train a flow following Eq.~\eqref{eq:ode_sfm}, with initial distribution $p(x_{t-1}|z_{\le t-1})$ and target distribution $p(x_{t}|z_{\le t})$.

\textbf{Training data construction.} Let us describe the naive training data construction protocol. Suppose we have a ground-truth trajectory $(x^*_{1:T},z^*_{1:T})$ drawn from an offline dataset. Formally, this trajectory is a sample $(x^{*}_{1:T},z_{1:T}^*)\sim p(x_{1:T},z_{1:T})$ from the joint distribution of states and observations. Suppose we want to train the filtering from $p(x_{t-1}|z_{\le t-1})$ to $p(x_t|z_{\le t})$, the naive approach extract $(x_{t-1}^*, x_t^*)$ and $z_{\le t}^*$ from the full trajectory. This means $(x_{t-1}^*, x_t^*,z_{\le t}^*)\sim p(x_{t-1},x_t,z_{\le t})$, i.e., they are drawn jointly from the marginal distribution.

Based on these settings, we can argue the following straightforward proposition, which shows the flow ODE's initial distribution is $p(x_{t-1}|z_{\le t})$ rather than the desired $p(x_{t-1}|z_{\le t-1})$.

\begin{proposition}
\label{thm:train_raw_data}
   Suppose a conditional flow matching is trained by minimizing
   \begin{equation}
    \theta^*=\arg\min_{\theta}\mathbb{E}_{(x^*_{t-1},x_t^*,z_{\le t}^*)\sim p(x_{t-1},x_t,z_{\le t})}\mathbb{E}_{\tau\sim \text{Unif(0,1)}} \norm{v_{\theta}(x_t(\tau),\tau;z^*_{\le t})-\dot{x}_\tau(t)}^2,
\end{equation}
where $x_t(\tau):=\alpha_\tau x_{t-1}^*+\sigma_{\tau}x^*_t$ and $\dot{x}_t(\tau)=\dot{\alpha}_{\tau}x_{t-1}^*+\dot{\sigma}_{\tau}x^*_{t}$.
Then given sufficient model capacity, the optimal velocity field $v_{\theta^*}(x_t(\tau), \tau;z_{\le t}^*)=\mathbb{E}_{x_{t-1},x_t|z_{\le t}}\dot{x}_{\tau}(t)$ induces a flow ODE from initial distribution $p(x_{t-1}|z_{\le t}^*)$ to $p(x_t|z_{\le t}^*)$.
\end{proposition}

\begin{proof}
     Let us factorize the distribution $p(x_{t-1},x_t,z_{\le t})=p(z_{\le t})p(x_{t},x_{t-1}|z_{\le t})$ and we rewrite the flow matching loss as
\begin{equation}
    \mathbb{E}_{z_{\le t}^*\sim p(z_{\le t})}\mathbb{E}_{(x^*_{t-1},x_t^*)\sim p(x_{t-1},x_t|z_{\le t})}\mathbb{E}_{\tau\sim \text{Unif(0,1)}} \norm{v_{\theta}(x_t(\tau),\tau;z^*_{\le t})-\dot{x}_\tau(t)}^2.
\end{equation}
The minimizer of the loss above is an expected loss over all conditions $z_{\le t}^*$. Therefore, the minimizer must also minimize a loss conditioned on any $z_{\le t}^*$, i.e., the loss below:
\begin{equation}
    \mathbb{E}_{(x^*_{t-1},x_t^*)\sim p(x_{t-1},x_t|z_{\le t})}\mathbb{E}_{\tau\sim \text{Unif(0,1)}} \norm{v_{\theta}(x_t(\tau),\tau;z^*_{\le t})-\dot{x}_\tau(t)}^2.
\end{equation}
The optimal solution to this loss leads to a conditional velocity $v(x_t(\tau),\tau;z_{\le t})=\mathbb{E}_{x_{t-1},x_t|z_{\le t}}\dot{x}_{\tau}(t)$, whose flow ODE transports from initial distribution $p(x_{t-1}|z_{\le t})$ to $p(x_t|z_{\le t})$. As a result, the initial distribution does not align with the previous posterior $p(x_{t-1}|z_{\le t-1})$.
\end{proof}



\section{Experimental Details}
\label{appx:exp}

\subsection{Dataset}
\label{appx:exp_dataset}
\begin{table}[ht!]
    \caption{Dataset statistics. The trajectory length refers to the total length of prediction for a specific task, and the prediction horizon is the length of rolling prediction window per physical time. The prediction horizon = "shrinking" means the task requires a shrinking prediction window whose length equals to the trajectory length at the beginning but decreases over time. $\#$pretraining and $\#$finetuning refer to the number of pretraining and finetuning trajectories.}
    \label{tab:dataset_stat}
    \resizebox{0.98\textwidth}{!}{
\begin{tabular}{lllllll}
    \toprule
Task & Dataset     & Trajectory length  & Prediction horizon & Feature dim.                  &  \#pretraining  & \#finetuning \\ \hline
\multirow{3}{*}{Forecasting} & Synthetic & 8 & 5 & 1 & 10,000 & 5,000  \\
& Beam Spill & 430 & 300 & 1 & 40,000 & 5,000 \\
& Burger's Equation & 16  & 10 & 64 & 90,000                                    & 10,000                                   \\
& WeatherBench2    &28 & 12  & 32$\times$32$\times$16    & 73,112                                    & 5,000                                    \\
\hline
State Estimation & Lorentz Attractor      & 100 & 1 & 3          & 75,000                                    & 5,000                                    \\\hline
\multirow{3}{*}{Planning\&Control} & Maze Medium & 600 & 600 (Shrinking) & 2 & 1,999,400 & 16,384 \\
& Maze Large & 800 & 800 (Shrinking) & 2 & 3,999,200 & 32,768\\
& Smoke Control  & 65 & 10 & 64$\times$64$\times$ 6 & 36,000 & 2,000\\ 
        \bottomrule
\end{tabular}
}
\end{table}

We present the dataset statistics in Table~\ref{tab:dataset_stat}. The trajectory length refers to the total prediction duration (or episode) and prediction horizon is the length of prediction window at each time step. $\#\text{pretrain}$ and $\#\text{finetuning}$ refer to the number of trajectories we use for model pretraining and finetuning. We provide details data construction below.

\textbf{Synthetic Bernoulli random walk process.} 
We construct a Bernoulli random walk
$p(s_{t+1}|s_t) = \tfrac{1}{2}\delta_{s_t+1} + \tfrac{1}{2}\delta_{s_t-1}$,
with initial state $s_1 \sim \mathcal{N}(0, 10^{-4})$. The models infer posterior distribution $p(s_{t+1:t+5}|s_t)$ over $t=1, 2,3,4$. Ground-truth $p(s_{t+\Delta}|s_t)$ is a discrete distribution with $(\Delta+1)$ modes. We first train a rectified flow model~\cite{liu2023flow} on 10,000 trajectories as the pretrained model, and then fine-tune it on 5,000 trajectories to obtain \model. We evaluate the accuracy of the learned distribution using the 1-Wasserstein distance
$\mathcal{W}_{1}(t) = \mathcal{W}_1\big(p_{\theta}(x_{t+1:t+5}|x_t), \, p(x_{t+1:t+5}|x_t)\big)$,
where $p_{\theta}$ denotes the empirical distribution induced by 1000 generated particles. We report the average Wasserstein distance
$\langle W_1 \rangle = \frac{1}{T} \sum_{t=1}^{T} W_1(t)$,
which summarizes performance across time steps ($T=4$). Figure~\ref{fig:w1_dist} plots $\langle W_1\rangle$ with respect to flow matching sampling steps. The results show that \model consistently achieves lower Wasserstein distance than the pretrained model under the same sampling budget. The uncertainty band is one time standard deviation of three runs of the models.

\textbf{Beam spill dynamic.} We adopt the simulator from~\citep{narayanan2022machine}. The state $s_t$ is a one-dimensional variable representing the beam intensity, and a PID controller takes the state and returns an one-dimensional control signal $a_t$. The simulator will simulate a 43 ms time window, with 0.1 ms per time step, resulting in a full state-action trajectory $x_{1:T},a_{1:T}$ with $T=430$. The simulator introduces a stochastic dynamic, by injecting a time-correlated noises $\epsilon_{1:T}$ for the process. As a result, it induces a probabilistic distribution $p(x_{1:T}, a_{1:T})$ where the future trajectory $x_{t+1:t+H}$ given history is uncertain.

\textbf{Burgers' equation.} We adopt the dataset from~\citep{wei2025cldiffphycon}. The 1D Burgers' equation follows:
\begin{equation}
\left\{
\begin{aligned}
&\frac{\partial s}{\partial t}
= -s \cdot \frac{\partial s}{\partial x}
   + \nu \frac{\partial^2 s}{\partial x^2}
   + a(x,t),
&& \text{in } [0,T]\times\Omega, \\
&s(x, t) = 0,
&& \text{in } [0,T]\times\partial\Omega, \\
&s(x, 0) = s_0(x),
&& \text{in } \{\tau=0\}\times\Omega.
\end{aligned}
\right.
\end{equation}
Here state states $s(x,t)$ is a field over space $x$ and time $t$, and input (action) $a(x,t)$ is also a function of space and time. This system dynamic is a deterministic process. To impose uncertainty, we introduce: (1) partial observation: only half of the space is observable, i.e., $s(x,t)$ for $x\in \Omega_{\text{right}}$ is removed from model input; (2) agnostic input: the input $a(x,t)$ (which is randomly generated in training and test dataset) is also excluded from model input.

\textbf{Weather forecasting.} We adopt the WeatherBench2~\citep{rasp2024weatherbench}. The dataset consists of global weather data from year 1959 to 2023, measured with 6 hours as interval. The data is a 2D/3D temporal data, including 2D features (surface variables) such as \texttt{sea\_level\_pressure}, \texttt{2m\_temperature}, and 3D features such \texttt{geopotential} as a function as height (measured by Atmospheric pressure). Due to computational resources, we only choose partial features (2D:\texttt{2m\_temperature}, \texttt{10m\_u\_component\_of\_wind}, \texttt{10m\_v\_component\_of\_wind}, \texttt{mean\_sea\_level\_pressure}, 3D: \texttt{geopotential},\texttt{temperature},\texttt{u\_component\_of\_wind},\texttt{v\_component\_of\_wind}, at height $500, 850, 1000$ Pa.) as our system states. We also constraint the latitude from 37.25 to 45.0 and longitude from 115.0 to 122.75. We use the data from 1959 to 2011 for pretraining, 2011 to 2015 for finetning and 2021 to 2023 for testing. We report \texttt{2m\_temperature} in our main table~\ref{tab:physics}.

\textbf{Maze planning.} We adopt maze-medium and maze-large from D4RL benchmark~\citep{fu2020d4rl}. The offline dataset consists of random-walk trajectories and at test time the goal is to reach a target position.

\textbf{Smoke Control.} We adopt the dataset from~\cite{wei2025cldiffphycon}. The system state is a 2D incompressible fluid following the Navier-Stokes equations:
\begin{equation}
\begin{cases}
\displaystyle
\frac{\partial \mathbf{v}}{\partial t}
+ \mathbf{v}\cdot\nabla \mathbf{v}
- \nu \nabla^2 \mathbf{v}
+ \nabla p
= \mathbf{f}, \\[6pt]
\nabla \cdot \mathbf{v} = 0, \\[6pt]
\mathbf{v}(\mathbf{x}, 0) = \mathbf{v}_0(\mathbf{x}).
\end{cases}
\end{equation}
Here the system state $s$ consists of velocity field $\mathbf{v}$ and pressure field $p$, and external force field $\mathbf{f}$ is the action. The task is to generate $\mathbf{f}$ to guide an initial smoko in the field to avoid obstacles and reach a target exit area. There are two settings in~\citep{wei2025cldiffphycon}: large domain control and boundary control. We adopt large domain control setting, where force signals are applied to all peripheral regions outside the obstacles, consisting of 1,792 cells.

\textbf{State Estimation.} We adopt the Lorenz attractor simulator from~\citep{he2025trackdiffuser} to generate data. The system state $s_t$ is a 3D vector following an nonlinear state-space model:
\begin{equation}
    s_t = \mathbf{F}(s_{t-1})s_{t-1} + w_t,
\end{equation}
\begin{equation}
\mathbf{F}(s_{t-1})
=
\exp\!\left(
\begin{bmatrix}
-10 & 10 & 0 \\
28 & -1 & -s_{t-1,1} \\
0 & s_{t-1,1} & -\dfrac{8}{3}
\end{bmatrix}
\, \Delta
\right),
\end{equation}
where $w_t$ is a noise term (we use Gaussian noise $w_t\sim \mathcal{N}(0, q^2I)$) and $\Delta$ is the time interval. The observation is $z_t=g(x_t)+\eta_t$, where $g(x_t)$ is a rotation operation and $\eta_t\sim \mathcal{N}(0, r^2I)$. The task is to estimation current state system $s_t$ given historical observation $z_{\le t}$ ($s_t$ cannot be directly observed).

\subsection{Implementation Details of Algorithms~\ref{alg:sfm_train},\ref{alg:sfm_inference}}
\label{appx:exp_our_method}

\textbf{Forecasting.}
In forecasting task we have $x_t= s_{t+1:t+H}$ and $z_t=x_t$, where $s_t$ is the real physical states we want to predict. To apply Algorithm~\ref{alg:sfm_train}, we take a small finetuning dataset $\{s^{(i)}_{1:T}, z^{(i)}_{1:T}\}_{i=1,2,...,n}$ and leverage a pretrained Rectified Flow model $v_{\theta_0}$ to generate some trajectories $\hat{s}^{(i)}_{t+1:t+H}\sim p_{\theta_0}(s_{t+1:t+H}|z^{(i)}_{\le t})$. We could directly use $\hat{s}^{(i)}_{t:t+H-1}$ as the source distribution and $s^{(i)}_{t+1:t+H}$ as the target distribution in \model ODE Eq.~\eqref{eq:ode_sfm}. In practice, to better align the prediction at the same physical time, we instead drop the first $\hat{s}^{(i)}_t$ (as it is already observed at time $t$) and pad a $\hat{s}^{(i)}_{t+H-1}$ (can be seen as a moving average) and construct $(\hat{s}^{(i)}_{t+1:t+H-1}, \hat{s}^{(i)}_{t+H-1})$ as the source distribution instead.

\textbf{Planning and Control.}
In planning and control task we have $x_{t}=(s_{t+1:t+H},a_{t:t+H-1})$ and $z_t=(s_t, a_{t-1})$. To apply Algorithm~\ref{alg:sfm_train}, we require an imitation learning dataset consisting of expert state-action trajectories. We treat the pretrained flow model $p_{\theta_0}$ as an expert policy and let it interact with the environment to collect a few state-action trajectories. Ultimately, we have model predicted trajectory $\hat{s}^{(i)}_{t+1:t+H},\hat{a}^{(i)}_{t:t+H-1}\sim p_{\theta_0}(s_{t+1:t+H},a_{t:t+H-1}|s^{(i)}_{1:t}, \hat{a}^{(i)}_{1:t-1})$ at any time $t$ and a resulting actual full trajectory $(s^{(i)}_{1:T}, a_{1:T}^{(i)})$, where $a^{(i)}_t$ is the actual executed action (i.e., the first predicted action at each time). We treat $(\hat{s}^{(i)}_{t:t+H-1},\hat{a}^{(i)}_{t-1:t+H-2})$ as the source distribution and $(s^{(i)}_{t+1:t+H},a^{(i)}_{t:t+H-1})$ as the target distribution in \model. Again, since at time $t$ we observe $z^{(i)}_t = (s^{(i)}_t, a^{(i)}_{t-1})$, in practice we remove already-observed state-action $\hat{s}^{(i)}_{t},\hat{a}^{(i)}_{t-1}$ and instead adopt $((\hat{s}^{(i)}_{t+1:t+H-1},\hat{s}^{(i)}_{t+H-1}), (\hat{a}^{(i)}_{t:t+H-2}, \hat{a}^{(i)}_{t+H-2})$ as the source distribution.

\textbf{State Estimation.} In state estimation task we have $x_t=s_t$. Similarly we call a pretrained model to generate $\hat{s}^{(i)}_{t}\sim p_{\theta_0}(s_{t}|z_{\le t}^{(i)})$. We use $\hat{s}^{(i)}_{t}$ as the source distribution and the actual physical state $s^{(i)}_t$ as the target distribution.

\subsection{Implementation Details of Baselines}
\label{appx:exp_baseline}


\textbf{Warm-start Diffusion.} The warm-start diffusion is a heuristic that denoise a previously noisy $x_{t-1}$, but conditioning on new observation $z_t$, to obtain updated prediction of $x_t$. In our implementation, warm-start diffusion share the exact pipeline as our \model, the only difference is that it use the pretrained model $v_{\mathrm{pre}}$ for denoising, while \model uses a finetuned flow $v_{\theta}$ for generation.

\textbf{Asynchronous Diffusion (CL-Diffusion~\citep{wei2025cldiffphycon}).} We implement asynchronous diffusion following the idea of \cite{wei2025cldiffphycon}. The original CL-Diffusion requires training two diffusion models, a synchronous diffusion and an asynchronous diffusion that allows different denoising scheduling. In our implementation, we use Diffusion Forcing~\citep{chen2024diffusion} as the pretrained model, which naturally allows to denoise in an arbitrary schedule (either synchronous or asynchronous). So we use a single pretrained diffusion forcing to replace the two models in CL-Diffusion.

\subsection{Diffusion Guidance in Maze Planning}

 The inference of diffusion models on Maze planning task often relies on diffusion guidance based on rewards, i.e., modifying local score/velocity function during sampling. The consistency model and MeanFlow baselines lack a well-defined score/velocity function, and therefore we omit these baselines on maze planning. On the other hand, \model are finetuned on the imitation learning dataset of trajectories produced by guided pretrained models. Therefore, the finetuned \model learns the transportation between guided trajectories, thus removing the need for explicit guidance in our implementation.

\section{Additional Experimental Results}
\label{appx:extra_exp}

\subsection{State Estimation}
\textbf{Datasets.} We consider state estimation of Lorenz attractor, a three-dimensional chaotic dynamical system following the nonlinear state-space model: $
    s_t = f(s_t)s_t + \mathcal{N}(0, q^2I)$ and $z_t = g(s_t)+\mathcal{N}(0, r^2I),$
where $s_t\in\mathbb{R}^3, f(s_t)\in\mathbb{R}^{3\times 3}, g(s_t)\in\mathbb{R}^3$. It is a simplified mathematical model used to capture chaotic behavior and to understand atmospheric convection \citep{lorenz2017deterministic}. We follow \cite{he2025trackdiffuser} to set $g$ be a rotation matrix operation, $q=r/10$ and varies $r$ to test the model performance under different environmental stochasticity. The model takes historical observations $z_{\le t}$ as conditions and predicts current state $x_t:=s_t$. The prediction horizon is $H=1$.
 
Except for learning-based approaches, we also compare against model-based approaches, including Extended Kalman Filtering~\citep{kushner1967approximations}, Unscented Kalman Filtering~\citep{julier1997new}, and Particle Filtering~\citep{del1997nonlinear}. These methods have explicit access to the underlying system dynamics $f(x_t)$ and measurement function $g(x_t)$, and therefore serve as oracle-style baselines.

\textbf{Results.} Table~\ref{tab:state} reports performance measured by $10\log_{10}(\text{MSE})$ under varying levels of environmental uncertainty (larger $1/r^2[\textbf{dB}]$ indicates lower stochasticity). As expected, the performance of all methods degrades as environmental stochasticity increases. Notably, warm-start approaches perform reasonably well when environmental stochasticity is low, but their performance becomes worse as uncertainty increases. This behavior is expected, since warm-start methods could approximate the correct predictive distribution when successive states $x_{t-1},x_t$ follow similar distributions, which is more likely to hold in near-deterministic system dynamics. In contrast, \model explicitly learns the probability flow across successive time steps. It consistently achieves competitive performance across all uncertainty levels with one sampling step among the learning-based methods.

\begin{table}[ht!]
    \centering
    \caption{Results of state estimation. The reported numbers are the $10\log_{10}(\text{MSE})$ (lower the better) with $\text{MSE}$ averaged over the entire episode (100 steps). $1/r^2[\textbf{dB}]=10\log_{10}(1/r^2)$ represents different uncertainty levels of measurements (smaller the $1/r^2[\textbf{dB}]$ larger the stochasticity). 
    }
    \label{tab:state}
    \resizebox{0.75\textwidth}{!}{
    \begin{tabular}{lccccc}
        \toprule
        \multirow{2}{*}{\textbf{Method}} & \multirow{2}{*}{\textbf{NFE}} & \multicolumn{4}{c}{Environment Stochasticity \textbf{$1/r^2$[dB]}}  \\
        &&-10 & 0 & 10 & 20  \\
        \midrule 
        \multicolumn{6}{l}{\textbf{Model-based}}\\ 
        Extended KF  & N/A &$2.69$ &$-6.19$ & $-16.49$ & $-25.18$  \\
        Unscented KF & N/A & $9.05$ & $2.58$ & $-3.52$ & $-16.24$  \\
        Particle Filtering & N/A &  $3.76$ & $-4.76$ & $-14.68$ & $-22.93$ \\
        \midrule 
        Diffusion Forcing  & 10 & $11.74_{\pm 0.01}$ & $0.68_{\pm 0.02}$ & $-9.57_{\pm 0.07}$ & $-17.08_{\pm 1.56}$ \\
        Rectified Flow & 5 & $11.93_{\pm 0.02}$ & $0.57_{\pm 0.02}$ & $-9.30_{\pm 0.30}$ & $-18.65_{\pm 0.33}$  \\ 
        Diffusion Forcing  & 1 & $10.59_{\pm 0.07}$ & $1.15_{\pm 0.17}$ & $-8.74_{\pm 0.30}$ & $-15.01_{\pm 2.07}$ \\
        Rectified Flow & 1 & $15.39_{\pm 0.10}$ & $4.90_{\pm 1.04}$ & $-5.94_{\pm 0.26}$ & $-13.25_{\pm 1.41}$ \\ 
        MeanFlow & 1 & $12.87_{\pm 0.62}$ & $6.20_{\pm 0.44}$ & $2.82_{\pm 2.00}$ & $0.85_{\pm 1.70}$ \\
        Consistency Model & 1 & $10.85_{\pm 0.00}$ & $4.15_{\pm 0.16}$ & $2.78_{\pm 0.15}$ & $2.45_{\pm 0.04}$  \\
        \midrule
        \multicolumn{6}{l}{\textbf{Warm-start Heuristic}}\\ 
        Warm-start Diffusion Forcing  & 1 &  $13.52_{\pm 0.03}$& $1.39_{\pm 0.17}$ & $-9.42_{\pm 0.04}$ & $-17.11_{\pm 1.51}$ \\
        Warm-start Rectified Flow & 1 & $16.86_{\pm 0.11}$ & $5.07_{\pm 0.40}$ & $-7.66_{\pm 0.42}$ & $-18.45_{\pm 0.48}$ \\
        \midrule  
        \multicolumn{6}{l}{\textbf{Ours}}\\ 
        SBFM-Rectified Flow & 1 & $10.12_{\pm 0.01}$ & $0.26_{\pm 0.20}$ & $-9.79_{\pm 0.06}$ & $-19.58_{\pm 0.09}$ \\
        \bottomrule
    \end{tabular}
    }
\end{table}

\subsection{Performance-Latency Trade-offs}
\label{appx:latency}
We further show the performance-latency trade-off on spill beam forecasting (Figure~\ref{fig:spill_time}) and maze planning (Figure~\ref{fig:maze_time}). We find \model only a few sampling steps ($\le 3$) to achieve superior performance. 

\begin{figure}[h]
    \centering
    \includegraphics[width=0.6\linewidth]{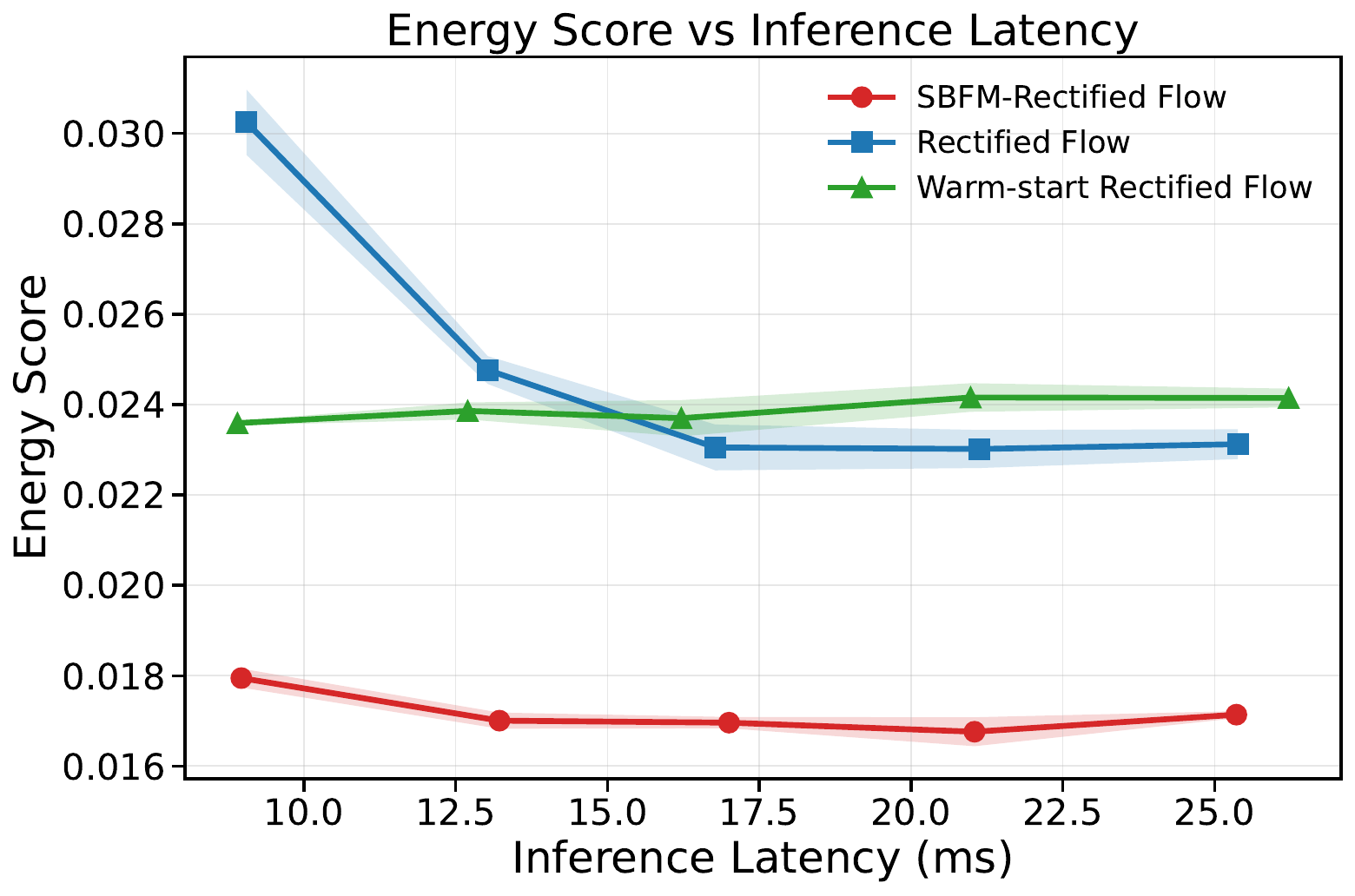}
    \caption{Inference latency (x axis) v.s. energy score (y axis) for beam spill forecasting under varying sampling steps (1, 2, 3, 4, 5 steps). }
    \label{fig:spill_time}
\end{figure}

\begin{figure}[h]
    \centering
    \includegraphics[width=0.6\linewidth]{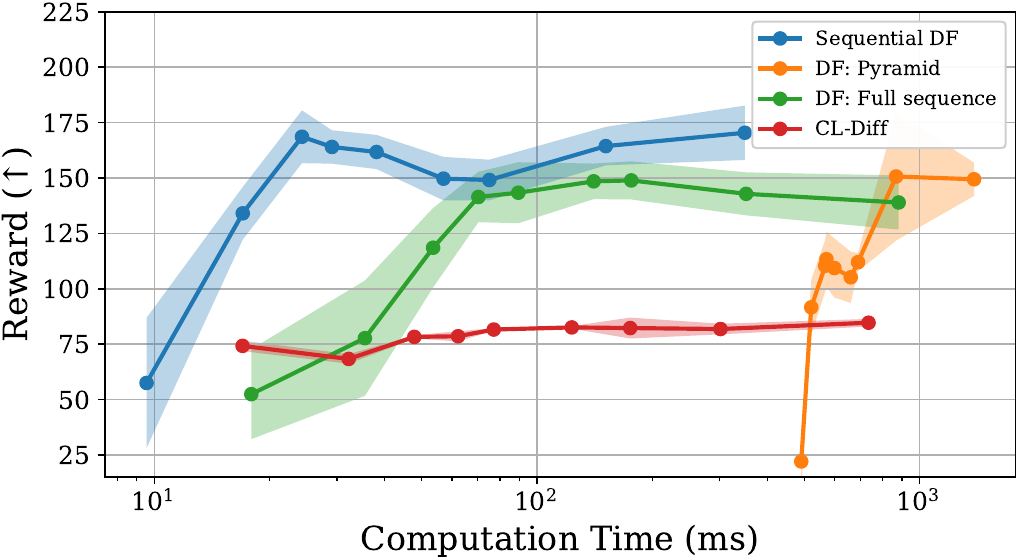}
    \caption{Inference latency (x axis) v.s. performance (y axis) for maze planning under varying sampling steps (1, 2, 3, 4, 5, 8, 10, 20 and 50 steps). DF: Pyramid and DF: Full-sequence refer to two denoising schedule of diffusion forcing. Sequential DF uses full sequence denoising by default.} 
    \label{fig:maze_time}
\end{figure}

\subsection{Ablation Study}
\label{appx:ablation}
We conduct ablation study on re-noise mechanisms and training with model-generated trajectories. 

\paragraph{Re-noise level.} Figure~\ref{fig:ab_renoise} reports performance under different re-noise levels $\tau_{\text{renoise}}$. We observe that both using a fully clean previous estimate ($\tau_{\text{renoise}}=0$) and completely discarding the previous estimate ($\tau_{\text{renoise}}=1$) lead to degraded performance. In contrast, there exists an intermediate range of $\tau_{\text{renoise}}$ that yields consistently strong and robust results. Moreover, this optimal range shifts toward larger values as system uncertainty increases (e.g., the optimal renoise level is $0.4\!\sim\!0.6$ for high system uncertainty $1/r^2$[dB]=-10, and is $0.2\!\sim\!0.4$ for low system uncertainty $1/r^2$[dB]=10). This behavior is expected: higher system uncertainty induces greater uncertainty in the model predictions, which in turn requires a higher re-noise level to adequately accommodate this stochasticity.

\begin{figure}[t]
    \centering
    \begin{minipage}{0.48\linewidth}
        \centering
        \includegraphics[width=\linewidth]{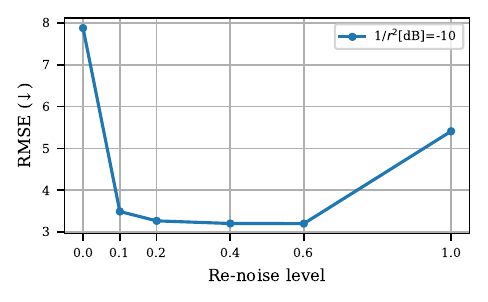}
    \end{minipage}\hfill
    \begin{minipage}{0.48\linewidth}
        \centering
        \includegraphics[width=\linewidth]{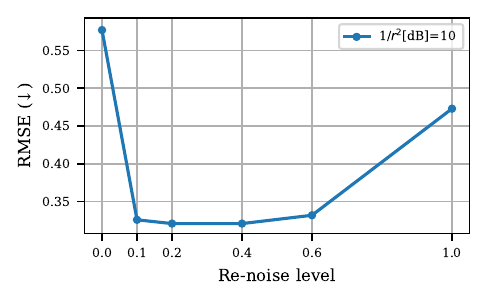}
    \end{minipage}
    \caption{Performance as a function of re-noise level $\tau_{\text{renoise}}$ on state estimation. Left and right figures are at different system uncertainty levels.}
    \label{fig:ab_renoise}
\end{figure}

\paragraph{Training with model generated trajectories.} Table~\ref{tab:ab} compares the performance of Burgers' equation forecasting and state estimation ($1/r^2$[dB]=-10) of using model-generated trajectories ($(\hat{x}_{t-1},x_t)$) against pure ground-truth trajectories $(x_{t-1},x_t)$ in \model finetuning. We see a significant performance degrade when exclusively using ground-truth trajectories on forecasting task, while comparable performance on state tracking. We hypothesis for long-horizon task like Burgers' equation forecasting, the model's prediction will be largely deviated from ground truth, while for short-horizon task like state estimation ($H=1$), the model prediction is close to ground truth. The different levels of train-test mismatch will decide if model-generated trajectories are necessary for finetuning \model.

\begin{table*}[ht!]
    \centering
    \caption{Ablation study of the choice of finetuning on model-generated trajectories or ground-truth trajectories.}
    \label{tab:ab}
    \resizebox{0.96\textwidth}{!}{
    \begin{tabular}{lcccc}
        \toprule
        \multirow{2}{*}{\textbf{Method}} &  \multicolumn{3}{c}{\textbf{Burgers' Equation}} &\multicolumn{1}{c}{\textbf{State Estimation}} \\
        & NFE & RMSE $\downarrow$ & Energy Score $\downarrow$ & $10\log_{10}\text{MSE}$ $\downarrow$ \\
        \midrule
        \model (finetuned from model-generated trajectories) & 1 & $0.239$ & $0.101$ & $10.12$   \\
        \model (finetuned from ground-truth trajectories) & 1 & $0.251$  & $0.108$  & $10.11$   \\\hline
    \end{tabular}
    }
\end{table*}

\section{Visualization}
We provide a visualization of maze planning (Figure~\ref{fig:maze}) and smoke control (Figure~\ref{fig:smoke}) using pretrained diffusion model and finetuned sequential diffusion model. 

\begin{figure*}[ht!]
    \centering
    \includegraphics[width=1\linewidth,
    trim=0 80 0 0,]{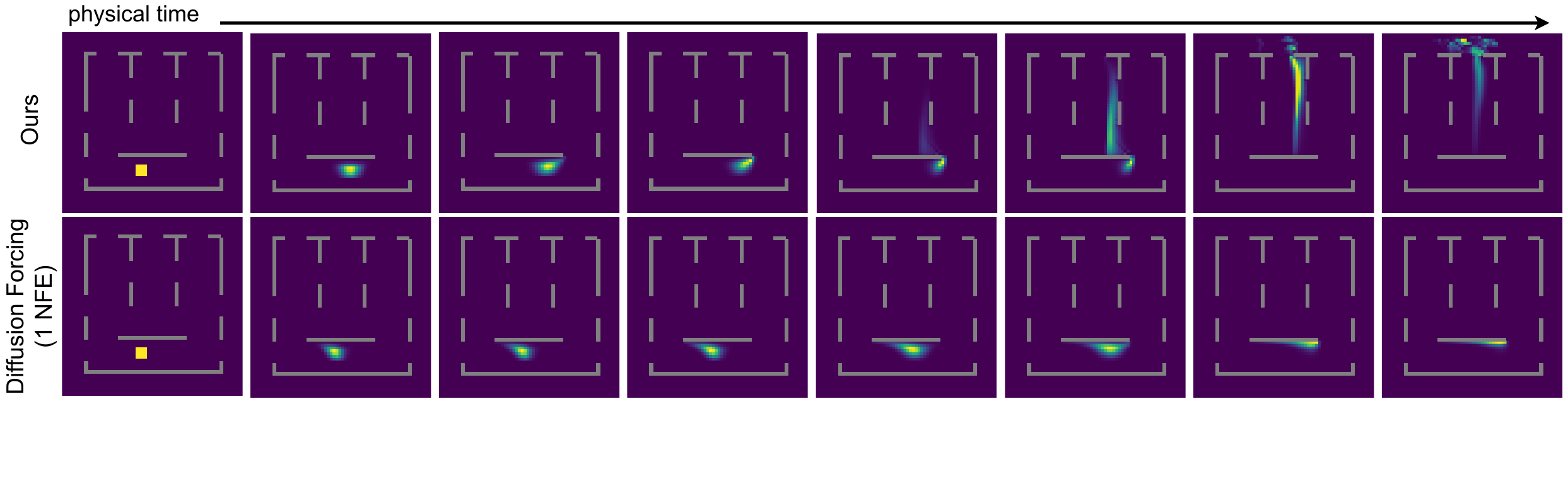}
    \caption{\textbf{Smoke Control.} With the same NFE, Sequential Diffusion Forcing effectively controls the smoke to reach the target exit, whereas Diffusion Forcing fails to circumvent the bottom obstacle.
    }
    \label{fig:smoke}
\end{figure*}

\begin{figure*}[ht!]
    \centering
    \includegraphics[width=1\linewidth]{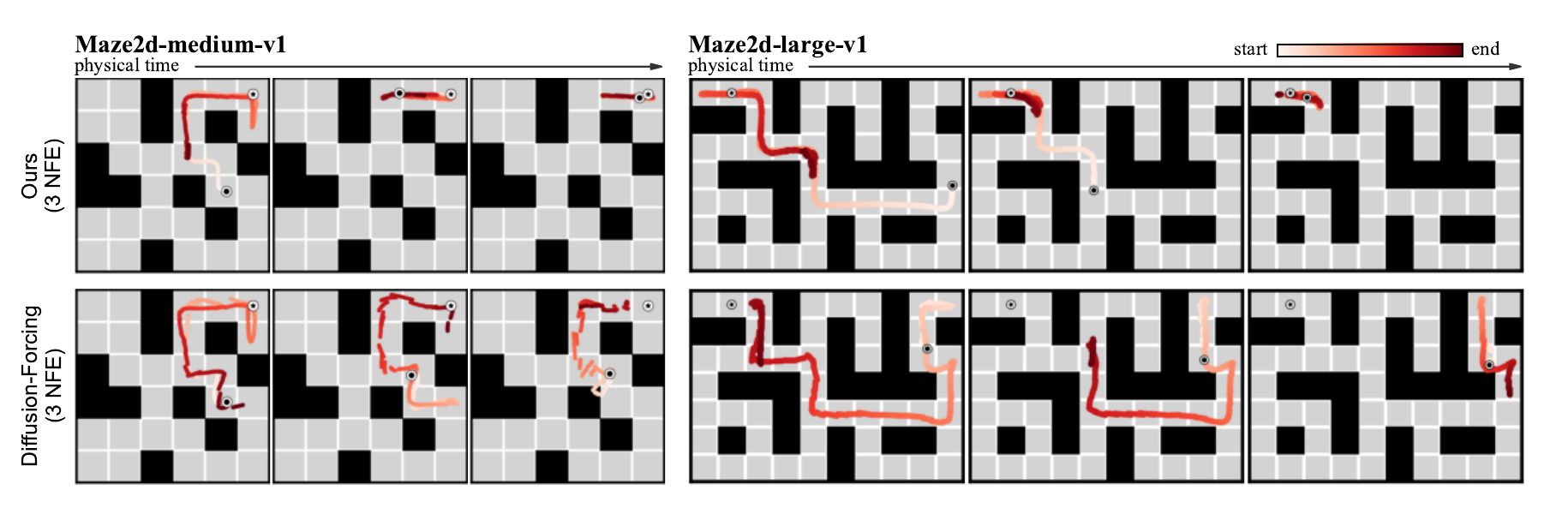}
    \caption{\textbf{Maze Planning.} With the same NFE, Sequential Diffusion Forcing can utilize the previous plans to effectively reach the target, while Diffusion Forcing fails drastically with small NFE.
    }
    \label{fig:maze}
\end{figure*}

\clearpage

\end{document}